\documentclass{article}

\usepackage{microtype}
\usepackage{graphicx}
\usepackage{subfigure}
\usepackage{booktabs} 

\usepackage{amsmath}
\usepackage{amsthm}
\usepackage{amsbsy,amsfonts,amssymb}
\usepackage{multirow}
\usepackage{wrapfig}
\usepackage{cite}

\usepackage{tgpagella}
\usepackage[utf8]{inputenc} 
\usepackage[T1]{fontenc}    
\usepackage{authblk}
\usepackage{algorithm}
\usepackage{algorithmic}
\usepackage[margin=0.95in]{geometry}

\newcommand*{\ISARXIV}{}%

\usepackage{hyperref}

\usepackage{cleveref}

\usepackage{enumitem}

\newcommand{\myparagraph}[1]{\noindent\textbf{#1.}}

\graphicspath{ {figs/} }

\def\0{\textbf{0}}
\def\1{\textbf{1}}
\def\a{\boldsymbol{a}}

\def\d{\boldsymbol{d}}
\def\e{\boldsymbol{e}}

\def\h{\boldsymbol{h}}
\def\v{\boldsymbol{v}}

\def\r{\boldsymbol{r}}
\def\s{\boldsymbol{s}}

\def\u{\boldsymbol{u}}
\def\v{\boldsymbol{v}}
\def\w{\boldsymbol{w}}
\def\x{\boldsymbol{x}}
\def\y{\boldsymbol{y}}
\def\z{\boldsymbol{z}}

\def\A{\boldsymbol{A}}

\def\J{\boldsymbol{J}}
\def\U{\boldsymbol{U}}
\def\V{\boldsymbol{V}}

\DeclareMathOperator{\diag}{diag}

\newcommand{\RR}{I\!\!R} 

\renewcommand{\mathbf}{\boldsymbol}

\newcommand{\mb}{\mathbf}
\newcommand{\mc}{\mathcal}

\newcommand{\bb}{\mathbb}

\newcommand{\paren}[1]{\left( #1 \right)}

\renewcommand{\Re}{{\mathbb R}}

\newcommand{\wh}{\widehat}

\hypersetup{
     colorlinks=true,%
     citecolor=blue,%
     filecolor=blue,%
     linkcolor=blue,%
     urlcolor=blue
}

 \usepackage[toc, page]{appendix}

\newcommand{\norm}[2]{\left\| #1 \right\|_{#2}}
\newcommand{\abs}[1]{\left| #1 \right|}
\newcommand{\innerprod}[2]{\left\langle #1,  #2 \right\rangle}

\theoremstyle{plain}
\newtheorem{theorem}{Theorem}[section]
\newtheorem{proposition}[theorem]{Proposition}
\newtheorem{lemma}[theorem]{Lemma}

\theoremstyle{definition}
\newtheorem{definition}[theorem]{Definition}

\theoremstyle{remark}

\title{Robust Training under Label Noise by Over-parameterization}

\author[$\sharp$]{Sheng Liu}
\author[$\natural$]{Zhihui Zhu}
\author[$\S$]{Qing Qu}
\author[$\dagger$]{Chong You}

\affil[$\sharp$]{Center for Data Science, New York University}
\affil[$\natural$]{Electrical and Computer Engineering, University of Denver}
\affil[$\S$]{Department of EECS, University of Michigan}
\affil[$\dagger$]{Google Research, New York City}

\begin{document}

\maketitle

\begin{abstract}
Recently, over-parameterized deep networks, with increasingly more network parameters than training samples, have dominated the performances of modern machine learning. However, when the training data is corrupted, it has been well-known that over-parameterized networks tend to overfit and do not generalize. 
In this work, we propose a principled approach for robust training of over-parameterized deep networks in classification tasks where a proportion of training labels are corrupted. 
The main idea is yet very simple: label noise is sparse and incoherent with
the network learned from clean data, so we model the noise and learn to separate it from the data. 
Specifically, we model the label noise via another \emph{sparse over-parameterization} term, and exploit implicit algorithmic regularizations to recover and separate the underlying corruptions. Remarkably, when trained using such a simple method in practice, we demonstrate state-of-the-art test accuracy against label noise on a variety of real datasets. 
Furthermore, our experimental results are corroborated by theory on simplified linear models, showing that exact separation between sparse noise and low-rank data can be achieved under incoherent conditions. 
The work opens many interesting directions for improving over-parameterized models by using sparse over-parameterization and implicit regularization\footnote{Code is available at: \url{https://github.com/shengliu66/SOP.}}. 
\end{abstract}

\section{Introduction}

One of the most important factors for the success of deep models is their large model size and high expressive power, which enable them to learn complicated input-output relations. 
As such, over-parametrized deep networks or large models, with more parameters than the size of training data, have dominated the performance in computer vision, natural language processing, and so on. 
The adoption of large models is justified by the recent discovery that deep models exhibit a ``double descent'' \cite{belkin2019reconciling} and ``uni-modal variance'' \cite{yang2020rethinking} generalization behavior, where their performance continues to improve beyond the interpolation point, extending the classical learning theory of bias-variance trade-off. 
While there are infinitely many global solutions that \emph{overfit} to training data, the choice of optimization algorithm imposes certain \emph{implicit} regularization \cite{neyshabur2014search} so that over-parameterized models converge to those that are generalizable. 

Nonetheless, the success of over-parameterization of deep networks critically depends on the availability of \emph{clean} training data, while overfitting inevitably occurs when training data is corrupted. 
Consider the task of image classification with a training dataset $\{(\x_i, \y_i)\}_{i=1}^N$, with $\x_i$ being an input image and $\y_i$ being the corresponding one-hot label. With an over-parameterized deep network $f(\cdot; \, \mb \theta)$, model training is achieved by solving an optimization problem with respect to (w.r.t.) the network parameter $\mb \theta$ as follows:
 \setlength{\belowdisplayskip}{3pt} \setlength{\belowdisplayshortskip}{3pt}
 \setlength{\abovedisplayskip}{3pt} \setlength{\abovedisplayshortskip}{3pt}
\begin{equation}\label{eq:model}
    \min_{\mb \theta}\; L(\mb \theta) \;=\; \frac{1}{N} \sum_{i=1}^N \ell\Big(f(\x_i; \, \mb \theta), \;\y_i\Big),
\end{equation}
where $\ell(\cdot, \cdot)$ is a loss function that measures the distance between network prediction $f(\x_i; \, \mb \theta)$ and the label $\y_i$. 
If a proportion of the images in the training set is \emph{mislabelled} \cite{song2020learning}, it is well-known that the network will be optimized to zero training error hence produce $f(\x_i; \mb \theta) \approx \y_i$ for all $i \in \{1,\cdots,N\}$, even for $\y_i$'s that are incorrect \cite{zhang2021understanding}.
Overfitting to wrong labels inevitably leads to poor generalization performance (see Fig.~\ref{fig:SOP_CE_ACC}).

\begin{wrapfigure}{r}{0.6\textwidth}
\centering  
\ifdefined\ISARXIV
\includegraphics[width=0.99\linewidth,trim={0.35cm 0 0.35cm 0},clip]{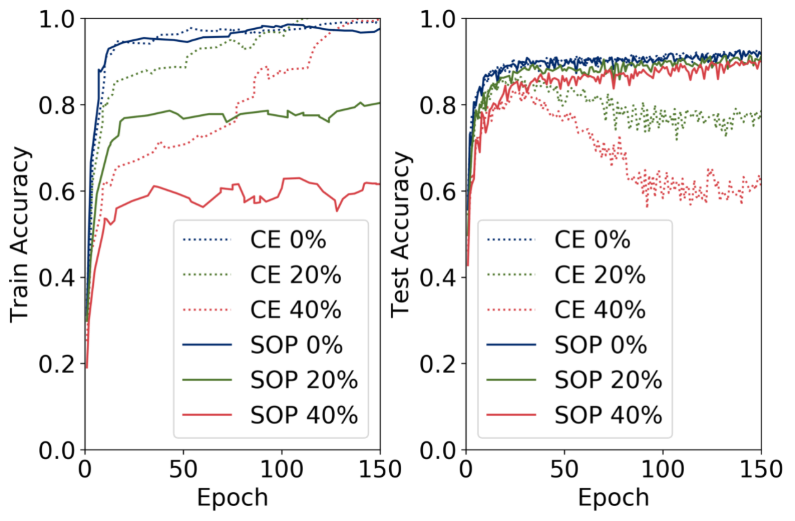}
\else
\includegraphics[width=0.99\linewidth,trim={0.35cm 0 0.35cm 0},clip]{figs/Overparam_Acc.png}
\fi
\vspace{-8pt}
\caption{\textbf{Sparse over-parameterization prevents overfitting to label noise.}
Training and test accuracy of a PreActResNet18 network trained with a standard cross entropy (CE) loss (dashed lines) and our Sparse Over-parameterization (SOP) (solid lines) for image classification on the CIFAR-10 dataset with 0\%, 20\%, and 40\% of the labels flipped at random. SOP prevents overfitting to the wrong training labels, obtaining near 100\%, 80\%, 60\% training accuracy respectively, therefore achieves better generalization on the test set without an accuracy drop at the end of training.}
\label{fig:SOP_CE_ACC}
\end{wrapfigure}
In this paper, we introduce a principled method to address the challenges of overfitting over-parameterized deep networks in the presence of training data corruptions. We focus on the task of classification trained with noisy label, a ubiquitous problem in practice due to the extreme complexity of data annotation even for experienced domain experts \cite{frenay2013classification}.
Our idea leverages the property that the label noise is \emph{sparse}, namely only a fraction of the labels are corrupted and the rest are intact.
Principled methods for dealing with sparse corruption have a rich history, which can be retraced back to compressed sensing \cite{candes2005decoding}, robust subspace recovery \cite{candes2011robust,wright2008robust}, and even earlier \cite{claerbout1973robust}.
Such methods are based on using a robust loss function, such as the $\ell_1$ norm which is less sensitive to large outlying entries.
While it is tempting to use sparse modeling for the label noise problem by setting the loss $\ell()$ in \eqref{eq:model} as the $\ell_1$ loss, such an approach cannot solve the overfitting issue since all global solutions are still given by those that satisfy $f(\x_i; \mb \theta) \approx \y_i$ for all $i \in \{1,\cdots,N\}$.
Hence, handling sparse corruptions with over-parameterized models requires the development of techniques beyond the classical $\ell_1$ loss for sparse modeling. 

\paragraph{Overview of our method and contribution.}
To handle sparse corruption with over-parameterized models, our idea is simply to use an extra variable $\mb s_i$ to model the unknown label noise $\mb s_{\star i}$, which is the difference between the observed label $\y_i$ and the corresponding clean label. 
Hence, the goal is to minimize the discrepancy between $f(\x_i;\mb \theta) + \s_i$ and $\y_i$. 
Inspired by a line of recent work \cite{vaskevicius2019implicit,zhao2019implicit,you2020robust}, we enforce sparsity of $\mb s_i$ by the over-parameterization $\mb s_i = \u_i \odot \u_i - \v_i \odot \v_i $ and optimize the following training loss
\begin{equation}\label{eq:main_model} 
    \min_{\mb \theta, \{\u_i, \v_i\}_{i=1}^N} L\left(\mb \theta, \{\u_i, \v_i\}_{k=1}^N\right), ~~\text{where}~~ L\left(\mb \theta, \{\u_i, \v_i\}_{k=1}^N\right)) \doteq \frac{1}{N} \sum_{i=1}^N \ell \left(f(\x_i;\mb \theta)  + \u_i \odot \u_i - \v_i \odot \v_i,\y_i \right)),
\end{equation}
with $\odot$ denoting an entry-wise Hadamard product. 
We term our method ``Sparse Over-Parameterization'' (SOP).
At the first glance, our SOP approach is seemingly problematic, because adding more learnable parameters $\{\u_i, \v_i\}_{i=1}^N$ to an over-parameterized network $f(\cdot, \mb \theta)$ would aggravate rather than alleviate the overfitting issue. 
Indeed, a global solution to \eqref{eq:main_model} is given by $\u_i \equiv \v_i \equiv \0$ and $f(\x_i, \mb \theta) \equiv \y_i$ for all $i \in \{1,\cdots,N\}$ where the network overfits to noisy labels. 
Here, we leverage the choice of a particular training algorithm to enforce an \emph{implicit bias} towards producing the desired solutions. 
Technically, we run gradient descent on the objective in \eqref{eq:main_model} starting from a small initialization for $\{\u_i, \v_i\}_{i=1}^N$:
\begin{equation}\label{eq:main_algorithm}
\begin{split}
    \mb \theta &\leftarrow \mb \theta - \tau \cdot \frac{\partial L(\mb \theta, \{\u_i, \v_i\})}{\partial \mb \theta},\\
    \u_i &\leftarrow \u_i - \alpha \tau \cdot \frac{\partial L(\mb \theta, \{\u_i, \v_i\})}{\partial \u_i},~~i=1, \ldots, N,\\
    \v_i &\leftarrow \v_i - \alpha \tau \cdot\frac{\partial L(\mb \theta, \{\u_i, \v_i\})}{\partial \v_i},~~i=1, \ldots, N,\\
\end{split}
\end{equation}
where $\alpha>0$ is the ratio of learning rates for different training variables.
Such a simple algorithm enables our method of SOP to train a deep image classification networks without overfitting to wrong labels and obtain better generalization performance (see Fig.~\ref{fig:SOP_CE_ACC}). 
A more comprehensive empirical study with a variety of datasets is presented in \Cref{sec:experiments}.




To rigorously justify our method, we theoretically investigate our method based upon a simplified over-parameterized linear model with sparse corruptions. As justified by a line of recent work \cite{jacot2018neural,chizat2018lazy}, over-parameterized linear models capture similar phenomena because they well approximate over-parameterized deep networks in a linearized regime around the initial points. Under sparse corruption and certain low-rank assumptions on the data, we show that the gradient descent \eqref{eq:main_algorithm} with an $\alpha$ below a certain threshold recovers the underlying model parameters with sparse corruptions. Our result is obtained by explicitly characterizing the implicit regularization for the term $\u_i \odot \u_i - \v_i \odot \v_i$. In particular, we explicitly show that it leads to an $\ell_1$-norm regularization on the sparse corruption, hence connecting our method to classical $\ell_1$ loss approaches for model robustness. For more details, we refer readers to \Cref{sec:theory}. 

In summary, our contributions are two-folds:
\ifdefined\ISARXIV
\begin{itemize}[topsep=0.3em,itemsep=0.1em]
\else
\begin{itemize}[leftmargin=*,topsep=0.0em,noitemsep]
\fi
    \item \emph{Method.} We proposed a simple yet practical SOP method that can effectively prevent overfitting for learning over-parameterized deep networks from corrupted training data, demonstrated on a variety of datasets.
    \item \emph{Theory.} Under a simplified over-parameterized linear model, we rigorously justify our approach for exactly separating sparse corruption from the data.
\end{itemize}
\ifdefined\ISARXIV
\else
\fi
Moreover, we believe the methodology we developed here could be far beyond the label noise setting, with the potential for dealing with more challenging scenarios of preventing overfitting in learning modern over-parametrized models of an ever-increasing size.

\label{sec:intro}


\section{Robust Classification with Label Noise}\label{sec:experiments}

In this section, we show how our SOP method plays out on image classification problems with the noisy label. In particular, we discuss extra implementation details of our method, followed by experimental demonstrations on a variety of datasets with synthetic and real label noise.

\subsection{Implementation Details of SOP}\label{subsec:implementation}
We train an over-parameterized deep neural network $f(\cdot; \mb \theta)$ from the noisy training data $\{(\x_i, \y_i)\}_{i=1}^N$ using the method briefly discussed in \Cref{sec:intro}.
Specifically, we train the network $f(\cdot; \mb \theta)$ using the objective \eqref{eq:main_model} with stochastic gradient descent (SGD) (i.e. a batch version of \eqref{eq:main_algorithm}).

Notice that there is additional prior information on label noise $\s_{\star i}$ associated with a sample $\{\x_i, \y_i\}$, namely, the positive and negative entries of $\s_{\star i}$ must correspond to nonzero entry and zero entries of $\y_i$, respectively.
Moreover, all entries of $\s_{\star i}$ must lie in the range of $[-1, 1]$.
To leverage such information, we optimize a variant of \eqref{eq:main_model} given by
\begin{gather}\label{eq:image_classification_model} 
    \min_{\mb \theta, \{\u_i, \v_i\}_{i=1}^N} \frac{1}{N} \sum_{i=1}^N \ell \Big( f(\x_i; \,\mb \theta)  +  \s_i, \,\y_i\Big), \\
    \text{s.t.}~~
    \s_i \doteq \u_i \odot \u_i \odot \y_i - \v_i \odot \v_i \odot (1-\y_i), ~~\text{and}~~ 
    \ifdefined\ISARXIV
    \else
    \\
    \fi
    \u_i \in [-1, 1]^K, ~~\v_i \in [-1, 1]^K,
\end{gather}
where $K$ is the number of classes. 
In above, constraints on $\u_i, \v_i$ are realized by performing a projection step after each gradient descent update.

\paragraph{Choice of the loss function $\ell(\cdot,\cdot)$ in \eqref{eq:image_classification_model}. }
The most commonly used loss function for classification tasks is the cross-entropy loss $\ell_{\mathrm{CE}}(\cdot,\cdot)$ \cite{krizhevsky2012imagenet}. 
Because the $\ell_{\mathrm{CE}}(\cdot,\cdot)$ loss requires a probability distribution as an input, we define a mapping
\begin{equation}
    \phi(\w) \doteq \frac{\max\{\w, \epsilon \1\}}{\|\max\{\w, \epsilon \1\}\|_1},
\end{equation}
and set the loss $\ell(\cdot,\cdot)$ in \eqref{eq:image_classification_model} to be
\ifdefined\ISARXIV
\begin{equation}\label{eq:ce-loss}
\ell \Big( f(\x_i; \,\mb \theta)  +  \s_i, \,\y_i\Big) = \ell_{\mathrm{CE}} \Big( \phi\big(f(\x_i; \,\mb \theta)  +  \s_i\big), \,\y_i\Big)
\end{equation}
\else
\begin{equation}\label{eq:ce-loss}
L_{\mathrm{CE}} \big(\mb \theta, \u_i, \v_i; \x_i, \y_i \big) \doteq \ell_{\mathrm{CE}} \Big( \phi\big(f(\x_i; \,\mb \theta)  +  \s_i\big), \,\y_i\Big).
\end{equation}
\fi

On the other hand, the cross-entropy loss cannot be used to optimize the variables $\{\v_i\}$ (see~\Cref{app:choice-of-loss} for an explanation). 
Hence, we use the mean squared error loss $\ell_{\mathrm{MSE}}$ and set the loss in \eqref{eq:image_classification_model} to be
\ifdefined\ISARXIV
\begin{equation}\label{eq:mse-loss}
    \ell \Big( f(\x_i; \,\mb \theta)  +  \s_i, \,\y_i\Big) = \ell_{\mathrm{MSE}} \Big( f(\x_i; \,\mb \theta)  +  \s_i, \,\y_i\Big), 
\end{equation}
\else
\begin{equation}\label{eq:mse-loss}
    L_{\mathrm{MSE}}\big( \mb \theta, \u_i, \v_i; \x_i, \y_i\big) \doteq \ell_{\mathrm{MSE}} \Big( f(\x_i; \,\mb \theta)  +  \s_i, \,\y_i\Big), 
\end{equation}
\fi
when optimizing $\{\v_i\}$\footnote{We also project $f(\x_i; \,\mb \theta)$ to a one-hot vector when using MSE loss which is empirically found to accelerate convergence of $\{\v_i\}$.}. 
We summarize our training method in \Cref{alg:algorithm}. 

\begin{algorithm}[t]
	\caption{Image classification under label noise by the method of Sparse Over-Parameterization (SOP).}
	\label{alg:algorithm}
	\begin{algorithmic}[1]
		\STATE \textbf{Input:} Training data $\{(\x_i, \y_i)\}_{i=1}^N$, network backbone $f(\cdot, \mb \theta)$, variables $\{\u_i, \v_i\}_{i=1}^N$, number of epochs $T$, learning rate $\tau$, learning rate ratio $\alpha_{u}, \alpha_{v}$, batch size $\beta$
		\STATE \textbf{Initialization:} Draw entries of $\u_i, \v_i$ from i.i.d. Gaussian distribution with zero-mean and s.t.d. $1e-8$
		\FOR {each $t \in \{1, \cdots, T\}$}
		    \STATE {\texttt{\# Train network $f(\cdot, \mb \theta)$ with SGD}}
		    \FOR {each $b \in \{1, \cdots, N / \beta\}$}
		        \STATE Sample a batch $\mathcal{B} \subseteq \{1, \ldots, N\}$ with $|\mathcal{B}|=\beta$
		       \ifdefined\ISARXIV
		       \STATE Set $\mb \theta \leftarrow \mb \theta - \tau \cdot \sum_{i \in \mathcal{B}}\frac{\partial \ell_{\mathrm{CE}} \big( \phi(f(\x_i; \,\mb \theta)  +  \s_i), \,\y_i\big)}{\partial \mb \theta}$		       
		       \else
		       \STATE Set $\mb \theta \leftarrow \mb \theta - \tau \cdot \sum_{i \in \mathcal{B}}\frac{\partial L_{\mathrm{CE} }(\mb \theta, \u_i, \v_i; \x_i, \y_i)}{\partial \mb \theta}$
		       \fi
		    \ENDFOR		
		    \STATE {\texttt{\# Update $\{\u_i, \v_i\}$}}		    
		    \FOR {each $i \in \{1, \cdots, N\}$}
		  \ifdefined\ISARXIV
		    \STATE Set $\s_i \leftarrow \u_i \odot \u_i \odot \y_i - \v_i \odot \v_i \odot (1-\y_i)$
		    \STATE Set $\u_i \leftarrow \mc P_{[-1,1]}\paren{ \u_i - \alpha_{u}\tau \cdot \frac{\partial \ell_{\mathrm{CE}} \big( \phi(f(\x_i; \,\mb \theta)  +  \s_i), \,\y_i\big)}{\partial \u_i} }$
		    \STATE Set $\v_i \leftarrow \mc P_{[-1,1]}\paren{\v_i - \alpha_{v}\tau \cdot \frac{\partial \ell_{\mathrm{MSE}} \big( f(\x_i; \,\mb \theta)  +  \s_i, \,\y_i\big)}{\partial \v_i}}$			  
		  \else
		    \STATE Set $\u_i \leftarrow \mc P_{[-1,1]}\paren{ \u_i - \alpha_{u}\tau \frac{\partial L_{\mathrm{CE} }(\mb \theta, \u_i, \v_i; \x_i, \y_i)}{\partial \u_i} }$
		    \STATE Set $\v_i \leftarrow \mc P_{[-1,1]}\paren{\v_i - \alpha_{v}\tau \frac{\partial L_{\mathrm{MSE} }(\mb \theta, \u_i, \v_i; \x_i, \y_i)}{\partial \v_i}}$	
		  \fi
		    \ENDFOR
		\ENDFOR
		\STATE \textbf{Output:} Network parameters $\mb \theta$ and $\{\u_i, \v_i\}_{i=1}^N$
		\end{algorithmic}
\end{algorithm}

\subsection{Experiments}\label{subsec:exp}

We experimentally demonstrate the effectiveness of our proposed SOP method on datasets with both synthetic (i.e., CIFAR-10 and CIFAR-100) and realistic (i.e., CIFAR-N, Clothing-1M, and WebVision) label noise.
In addition to the SOP described in~\Cref{alg:algorithm}, we also implement an improved version, termed SOP+, which incorporates two commonly used regularization techniques in the literature of label noise, namely the consistency regularization and the class-balance regularization.
We explain SOP+ in more detail in Appendix~\ref{app:sop+}.

\paragraph{Dataset descriptions.}
We use datasets with synthetic label noise generated from CIFAR-10 and CIFAR-100~\cite{krizhevsky2009learning}.
Each dataset contains 50k training images and 10k test images, all with clean labels, where each image is of size $32 \times 32$. Following previous works~\cite{NEURIPS2018_a19744e2, liu2020early, xia2020robust}, we generate symmetric label noise by uniformly flipping labels for a percentage of the training set for all classes, as well as asymmetric label noise by flipping labels for particular pairs of classes. 
For datasets with realistic label noise, we test on CIFAR-10N/CIFAR-100N \cite{wei2021learning} which contains a re-annotation of CIFAR-10/CIFAR-100 with human workers. 
Specifically, each image in CIFAR-10N contains three submitted labels (i.e., \textbf{Random {1, 2, 3}}) which are further combined to have an \textbf{Aggregate} and a \textbf{Worst} label. Each image in CIFAR-100N contains a single submitted label for the fine classes. 
We also test on Clothing-1M~\cite{xiao2015learning} which is a large-scale dataset with images clawed from online shopping websites and labels generated based on surrounding texts.  
Clothing-1M contains 1 million training images, 15k validation images, and 10k test images with clean labels. 
Finally, we also test on the mini WebVision dataset~\cite{li2017webvision} which contains the top 50 classes from the Google image subset of WebVision (approximately 66 thousand images). 
Models trained on mini WebVision are evaluated on both WebVision and ImageNet ILSVRC12 validation set.
Details on the label noise for these datasets is provided in~\Cref{app:noise-definition}.

\paragraph{Network structures \& hyperparameters.}  
We implement our method with PyTorch v1.7. For each dataset, the choices of network architectures and hyperparameters for SOP are as follows. Additional details, as well as hyper-parameters for both SOP and SOP+, can be found in Appendix~\ref{app:experiment}.

\ifdefined\ISARXIV
\begin{itemize}[topsep=0.3em,itemsep=0.1em]
\else
\begin{itemize}[leftmargin=*,topsep=0.0em,noitemsep]
\fi
    \item \emph{CIFAR-10/100 and CIFAR-10N/100N}. We follow~\cite{liu2020early} to use ResNet-34 and PreActResNet18 architectures trained with SGD using a 0.9 momentum. The initial learning rate is $0.02$ decayed with a factor of 10 at the 40th and 80th epochs for CIFAR-10/CIFAR-10N and at the 80th and 120th epochs for CIFAR-100/CIFAR-100N, respectively. Weight decay for network parameters $\mb \theta$ is set to $5\times 10^{-4}$. 
    No weight decay is used for parameters $\{\u_i, \v_i\}_{i=1}^N$.
    \item \emph{Clothing-1M}. We follow the previous work~\cite{liu2020early} to use a ResNet-50~\cite{he2016deep} pre-trained on ImageNet~\cite{krizhevsky2012imagenet}. The network is trained with batch size 64 and an initial learning rate 0.001, which is reduced by a factor of $10$ after 5th epoch (10 training epochs in total). Optimization is performed using SGD with a momentum 0.9. 
    Weight decay is 0.001 for parameters $\mb \theta$ and is zero for parameters $\{\u_i, \v_i\}_{i=1}^N$.
    \item \emph{Mini Webvision}. We use InceptionResNetV2 as the backbone architecture. All other optimization details are the same as for CIFAR-10, except that we use weight decay $0.0005$ and batch size $32$.
\end{itemize}

\begin{table*}[t]
\centering
\caption{\textbf{Test accuracy with synthetic label noise} on CIFAR-10 and CIFAR-100 with $\{20\%, 40\%, 60\%, 80\%\}$ percent of labels for training data randomly flipped uniformly to another class. All methods use ResNet34 as the architecture. Mean and standard deviation over 5 independent runs are reported. }
\resizebox{\textwidth}{!}{%
\begin{tabular}{l|cccc|cccc}
\toprule
\multirow{2}{*}{Methods}  & \multicolumn{4}{c}{CIFAR-10} & \multicolumn{4}{|c}{CIFAR-100} \\
 & 20\% & 40\% & 60\% & 80\% & 20\% & 40\% & 60\% & 80\%\\
\midrule
CE & $86.32_{\pm 0.18}$  & $82.65_{\pm 0.16} $ & $76.15_{\pm 0.32} $ & $59.28_{\pm 0.97}$ & $51.43_{\pm 0.58}$ & $45.23_{\pm 0.53}$ & $36.31_{\pm 0.39} $ & $20.23_{\pm 0.82}$ \\
Forward & $87.99_{\pm 0.36}$ & $83.25_{\pm 0.38}$ & $74.96_{\pm 0.65}$ & $54.64_{\pm 0.44}$ & $39.19_{\pm 2.61}$ & $31.05_{\pm 1.44}$ & $19.12_{\pm 1.95}$ & $8.99_{\pm 0.58}$ \\
GCE   & $89.83_{\pm 0.20}$ & $87.13_{\pm 0.22}$ & $82.54_{\pm 0.23}$ & $64.07_{\pm 1.38}$ & $66.81_{\pm 0.42}$ & $61.77_{\pm 0.24}$ & $53.16_{\pm 0.78}$ & $29.16_{\pm 0.74}$ \\
SL & $89.83_{\pm 0.32}$ & $87.13_{\pm 0.26}$ & $82.81_{\pm 0.61}$  & $68.12_{\pm 0.81}$ & $70.38_{\pm 0.13}$ & $62.27_{\pm 0.22}$ & $54.82_{\pm 0.57}$ & $25.91_{\pm 0.44}$ \\
ELR& $91.16_{\pm 0.08}$ & $89.15_{\pm 0.17}$ & $86.12_{\pm 0.49}$ & $\textbf{73.86}_{\pm 0.61}$ & $74.21_{\pm 0.22}$ &  $68.28_{\pm 0.31}$ & $59.28_{\pm 0.67}$ & $29.78_{\pm 0.56}$ \\
\midrule
SOP (ours) & $\textbf{93.18}_{\pm 0.57}$ & $\textbf{90.09}_{\pm 0.27}$ & $\textbf{86.76}_{\pm 0.22}$ & $68.32_{\pm 0.77}$ & $\textbf{74.67}_{\pm 0.30}$ & $\textbf{70.12}_{\pm 0.57}$ & $\textbf{60.26}_{\pm 0.41}$ & $\textbf{30.20}_{\pm 0.63}$ \\
\bottomrule
\end{tabular}
}
\label{tab:cifar}
\end{table*}



\begin{table}[t]
\centering
\caption{\textbf{Comparison with the state-of-the-art methods} that use two network ensembles and semi-supervised learning on CIFAR-10 and CIFAR-100 under symmetric (with ${20\%, 50\%, 80\%}$) and asymmetric (with $40\%$) label noise. All methods use ResNet34 as the architecture. }
\ifdefined\ISARXIV
\begin{tabular}{l|ccc|c|ccc|c}
\toprule
\multirow{3}{*}{Methods}  & \multicolumn{4}{c}{CIFAR-10} & \multicolumn{4}{|c}{CIFAR-100} \\
 & \multicolumn{3}{c}{Symmetric} & Asym &\multicolumn{3}{c}{Symmetric} & Asym \\
\else
\resizebox{\columnwidth}{!}{%
\begin{tabular}{@{\;}l@{\;}|@{\;}c@{\;\;}c@{\;\;}c@{\;}|@{\;}c@{\;}|@{\;}c@{\;\;}c@{\;\;}c@{\;}|@{\;}c@{\;}}
\toprule
 & \multicolumn{4}{c}{CIFAR-10} & \multicolumn{4}{@{\negthickspace}|c}{CIFAR-100} \\
 \midrule
 & \multicolumn{3}{c}{Symmetric} & Asym &\multicolumn{3}{@{\;}c}{Symmetric} & Asym \\
\fi
     & 20\% & 50\% & 80\%  & 40\% & 20\% & 50\% & 80\%  & 40\% \\
\midrule
CE & $87.2$  & $80.7$ & $65.8 $ & $82.2$  & $ 58.1$ & $47.1$ & $23.8$ & $43.3$\\
MixUp & $93.5$ & $87.9$ & $72.3$ & - & $69.9$ & $57.3$ & $33.6$ & - \\
DivideMix & $96.1$ & $94.6$ & $93.2$ & $93.4$ & $77.1$ & $74.6$ & $60.2$ & $72.1$\\
ELR+ & $95.8$ & $94.8$ & $ 93.3$ & $93.0$ & $77.7$ & $ 73.8$ & $60.8$ & $77.5$ \\ 
\midrule
SOP+ (ours) & $\textbf{96.3}$ & $\textbf{95.5}$ & $\textbf{94.0}$  & $\textbf{93.8}$ & $\textbf{78.8}$ & $\textbf{75.9}$ & $\textbf{63.3}$ & $\textbf{78.0}$\\
\bottomrule
\end{tabular}
\ifdefined\ISARXIV
\else
}
\fi
\label{tab:cifar_with_regularization}
\end{table}


\paragraph{Experimental results.}
We compare with methods based on estimation of the transition matrix (Forward~\cite{patrini2017making}), design of loss functions
(GCE~\cite{zhang2018generalized} and SL~\cite{wang2019symmetric}), training two networks (Co-teaching~\cite{NEURIPS2018_a19744e2} and DivideMix~\cite{li2020dividemix}), and label noise correction (ELR~\cite{liu2020early} and CORES$^2$~\cite{cheng2021learning}).

\Cref{tab:cifar} reports the performance of our method on synthetically generated symmetric label noise using CIFAR-10 and CIFAR-100. 
To compare with state-of-the-art methods, we also report the performance of SOP+ which contains additional regularization on both symmetric and asymmetric label noise and report the results in Table~\ref{tab:cifar_with_regularization}. 
It can be observed that our method is robust to a fairly large amount of label noise, and compares favorably to existing techniques. 

We further demonstrate that our method can effectively handle datasets with realistic label noise by reporting its performance on Clothing1M \& WebVision (see Table~\ref{tab:clothing_and_webvision}) and CIFAR-N (see Table~\ref{tab:cifar-n}) datasets. 
We can observe a performance gain over all comparing methods. 


Finally, we compare the training time (on a single Nvidia V100 GPU) of our method to the baseline methods in~\Cref{tab:training_time}. 
We observe that our algorithm SOP/SOP+ achieves the fastest speed across all baselines.



\begin{table}[t]
\centering
\caption{\textbf{Test accuracy with realistic label noise on Clothing1M and WebVision.} We use a pre-trained ResNet50 for Clothing1M and an InceptionResNetV2 for WebVision dataset. The results of the comparing methods are taken from their respective papers.}
\begin{tabular}{l|c|c|c}
\toprule
      Methods & Clothing1M &WebVision & ILSVRC12 \\
      \midrule
      CE & $69.1$ & - & - \\
      Forward  & $69.8$ & $61.1$ & $57.3$\\
      Co-teaching & $69.2 $ & $63.6$ & 61.5\\
      ELR &  $72.9$ & $76.2$ & 68.7\\
      CORES$^2$ & $73.2$ & - & -\\
      \midrule
      SOP (ours) & $\mathbf{73.5}$ & $\mathbf{76.6}$ & $\mathbf{69.1}$\\
      \bottomrule
\end{tabular}
\label{tab:clothing_and_webvision}
\end{table}

\begin{table}[t]
\centering
\caption{\textbf{Test accuracy with realistic label noise on CIFAR-N.} Mean and standard deviation over 5 independent runs are reported. The results of the baseline methods are taken from \cite{wei2021learning} which all use ResNet34 as the architecture. For SOP+, we use PreActResNet18. }
\resizebox{\textwidth}{!}{%
\begin{tabular}{l|cccccc|cc}
\toprule
      \multirow{2}{*}{Methods} & \multicolumn{6}{c|}{CIFAR-10N} & \multicolumn{2}{c}{CIFAR-100N} \\
             & Clean & Random 1 & Random 2 & Random 3 & Aggregate & Worst & Clean & Noisy \\
      \midrule
      CE & $92.92_{\pm 0.11}$ & $85.02_{\pm 0.65}$ & $86.46_{\pm 1.79}$ & $85.16_{\pm 0.61}$ & $87.77_{\pm 0.38}$ & $77.69_{\pm 1.55}$ & $76.70_{\pm 0.74}$ & $55.50_{\pm 0.66}$\\
      Forward & $93.02_{\pm 0.12}$ & $86.88_{\pm 0.50}$ & $86.14_{\pm 0.24}$    & $87.04_{\pm 0.35}$ & $88.24_{\pm 0.22}$ & $79.79_{\pm 0.46}$ & $76.18_{\pm 0.37}$ & $57.01_{\pm 1.03}$    \\
      Co-teaching & $93.35_{\pm 0.14}$ & $90.33_{\pm 0.13}$ & $90.30_{\pm 0.17}$ & $90.15_{\pm 0.18}$ & $91.20_{\pm 0.13}$ & $83.83_{\pm 0.13}$ & $73.46_{\pm 0.09}$ & $60.37_{\pm 0.27}$\\ 
      ELR+ & $95.39_{\pm 0.05}$ & $94.43_{\pm 0.41}$ & $94.20_{\pm 0.24}$ & $94.34_{\pm 0.22}$ & $94.83_{\pm 0.10}$ & $91.09_{\pm 1.60}$ & $78.57_{\pm 0.12}$ & $66.72_{\pm 0.07}$ \\
      CORES$^*$ & $94.16_{\pm 0.11}$ & $94.45_{\pm 0.14}$ & $94.88_{\pm 0.31}$ & $94.74_{\pm 0.03}$ & $95.25_{\pm 0.09}$ & $91.66_{\pm 0.09}$ & $73.87_{\pm 0.16}$ & $55.72_{\pm 0.42}$\\
      \midrule
      SOP+(ours) & $\mathbf{96.38}_{\pm 0.31}$ & $\mathbf{95.28}_{\pm 0.13}$ & $\mathbf{95.31}_{\pm 0.10}$ & $\mathbf{95.39}_{\pm 0.11}$ & $\mathbf{95.61}_{\pm 0.13}$ & $\mathbf{93.24}_{\pm 0.21}$ & $\mathbf{78.91}_{\pm 0.43}$ & $\mathbf{67.81}_{\pm 0.23}$\\
      \bottomrule
\end{tabular}
}
\label{tab:cifar-n}
\end{table}


\begin{table}[t]
\ifdefined\ISARXIV
\else
\footnotesize
\fi
\centering
\caption{\textbf{Comparison of total training time} in hours on CIFAR-10 with 50\% symmetric label noise.}
\begin{tabular}{c|c|c|c|c|c}
\toprule
CE & Co-teaching+ & DivideMix & ELR+ & SOP (ours) & SOP+ (ours)\\
\midrule
0.9h & 4.4h & 5.4h & 2.3h & 1.0h & 2.1h \\
\bottomrule
\end{tabular}
\label{tab:training_time}
\ifdefined\ISARXIV
\else
\fi
\end{table}







\section{Theoretical Insights with Simplified Models}\label{sec:theory}
 This section provides theoretical insights into our SOP method by studying structured data recovery with sparse corruption in the context of over-parameterized \emph{linear} models. We will start with model simplification, followed by our main theoretical results and experimental verification.


\subsection{Problem Setup \& Main Result}
Given a highly \emph{overparameterized} network $f(\cdot; \mb \theta)$, recent work \cite{jacot2018neural,kalimeris2019sgd} suggests that 
the parameter $\mb \theta \in \bb R^p$ may not change much from its initialization $\mb \theta_0$ before obtaining zero training error.
Hence, a nonlinear network $f(\cdot; \mb \theta): \bb R^n \mapsto \bb R$ can be well approximated by its first-order Taylor expansion:
\begin{equation}
\begin{split}
    f(\x; \mb \theta) \approx f(\x; \mb \theta_0) + \langle \nabla_{\mb \theta} f(\x; \mb \theta_0), \;\mb \theta - \mb \theta_0\rangle,
\end{split}
\end{equation}
where we consider $f(\cdot; \mb \theta)$ as a scalar function for simplicity. Since the bias term $f(\x; \mb \theta_0) - \langle \nabla_{\mb \theta} f(\x; \mb \theta_0), \; \mb \theta_0\rangle$ is constant w.r.t. $\mb \theta$, for simplicity we may further assume that 
\begin{equation}
    f(\x; \mb \theta) \approx \langle \nabla_{\mb \theta} f(\x; \mb \theta_0), \;\mb \theta\rangle.
\end{equation}


Thus, for a dataset $\{\x_i\}_{i=1}^N$ of $N$ points, collectively 
\begin{equation}
\begin{split}
    \begin{bmatrix} f(\x_1; \mb \theta) \\ \vdots \\ f(\x_N; \mb \theta) \end{bmatrix}  
    \;&\approx\; 
    \begin{bmatrix} \nabla_{\mb \theta} f(\x_1; \mb \theta_0) \\ \vdots \\ \nabla_{\mb \theta} f(\x_N; \mb \theta_0) \end{bmatrix} \cdot  \mb \theta
    \;=\; \J \cdot \mb \theta,
\end{split}
\end{equation}
where $\J \in \RR^{N \times p}$ is a Jacobian matrix.
This observation motivates us to consider the following problem setup.

\paragraph{Problem setup.}
Based upon the above linearization, we assume that our corrupted observation $\mb y\in \bb R^N$ (e.g., noisy labels) is generated by
\begin{equation}\label{eq:model-linear}
    \y = \J \cdot \mb \theta_\star + \s_\star,
\end{equation}
where $\mb \theta_\star\in \bb R^p$ is the underlying groundtruth parameter, and the noise $\mb s_\star \in \bb R^N$ is \emph{sparse} so that only a subset of observation (e.g., labels) is corrupted. Given $\mb J$ and $\mb y$ generated from \eqref{eq:model-linear}, our goal is to recover both $\mb \theta_\star$ and $\mb s_\star$.


However, as we are considering the problem in an over-parameterized regime with $p > N$, the underdetermined system \eqref{eq:model-linear} implies that there are \emph{infinite} solutions for $\mb \theta_\star$ even if $\mb s_\star$ is given. Nonetheless, recent work showed that the implicit bias of gradient descent for overparameterized linear models and deep networks tend to find minimum $\ell_2$-norm solutions  \cite{zhang2021understanding}. To make our problem more well-posed, motivated by these results, we would like to find an $\mb \theta_\star$ with minimum $\ell_2$-norm, namely, 
\begin{equation}\label{eq:property-of-theta}
    \mb \theta_\star = \arg\min_{\mb \theta} \|\mb \theta\|_2^2   \quad \text{s.t.} \quad \y = \J  \mb\theta + \s_\star.
\end{equation}

Analogous to \eqref{eq:main_model}, we will show that $\mb \theta_\star$ and $\s_\star$ can be provably recovered by solving the problem
\begin{equation}\label{eq:objective-linear}
    \min_{\mb\theta, \u, \v} h(\mb \theta, \u, \v) \doteq \frac{1}{2} \| \J \mb\theta + \u \odot \u - \v \odot \v - \y \|_2^2,
\end{equation}
using the gradient descent algorithm with learning rates $\tau$ and $\alpha \tau$ on $\mb \theta$ and $\{\u, \v\}$, respectively:
\begin{equation}\label{eq:gradient-descent-linear}
\begin{split}
    \mb \theta_{k+1} &= \mb\theta_k - \tau \cdot \J^\top \r_k, \\
    \u_{k+1} &= \u_k - 2 \alpha \tau \cdot \u_k \odot \r_k, \\
    \v_{k+1} &= \v_k + 2 \alpha \tau \cdot \v_k \odot \r_k,
\end{split}
\end{equation}
where $\r_k \doteq \J \mb \theta_k + \u_k \odot \u_k - \v_k \odot \v_k - \mb y$. Based on these, our result can be summarized as follows.

\begin{theorem}[Main result, informal]\label{thm:main}
Suppose $\mb J$ is rank-$r$ and $\mu$-incoherent defined in \Cref{sec:exact-recovery}, and $\mb s_\star$ is $k$-sparse. If $k^2r<N/(4\mu)$, with $\tau \rightarrow 0$ and a proper choice of $\alpha$ depending on $(\mb J, \mb \theta_\star,k)$, the gradient dynamics of \eqref{eq:gradient-descent-linear} converges to the ground truth solution $(\mb \theta_\star,\mb s_\star)$ in \eqref{eq:model-linear} starting from a small initialization of $(\mb \theta, \mb u, \mb v)$.
\end{theorem}
We state our result at a high level with more technical details in \Cref{sec:landscape} and \Cref{sec:exact-recovery}. The overall idea of the proof can be sketched through the following two steps.
\ifdefined\ISARXIV
\begin{itemize}[topsep=0.3em,itemsep=0.1em]
\else
\begin{itemize}[leftmargin=*,topsep=0.0em,noitemsep]
\fi
    \item First, although the problem \eqref{eq:objective-linear} is \emph{nonconvex}, in \Cref{sec:landscape} we show that it has benign global landscape, and that the gradient descent \eqref{eq:gradient-descent-linear} converges to particular global solutions that are the same as solutions to a \emph{convex} problem with explicit regularizations on $\mb \theta$ and $\s$.
    \item  Building upon above results, in \Cref{sec:exact-recovery} we complete our analysis by showing that $\mb \theta_\star$ and $\s_\star$ can be exactly recovered by the convex problem with a small enough value for $\alpha$.
\end{itemize}
Throughout the analysis, we corroborate our findings with numerical simulations.

\subsection{Landscapes \& Implicit Sparse Regularization}
\label{sec:landscape}

\paragraph{Benign global landscape.} We start by characterizing the nonconvex landscape of \eqref{eq:objective-linear}, showing the following result.

\begin{proposition}\label{thm:landscape}
Any critical point of \eqref{eq:objective-linear} is either a global minimizer, or it is a strict saddle \cite{ge2015escaping} with its Hessian having at least one negative eigenvalue.
\end{proposition} 
\ifdefined\ISARXIV
\else
\fi

For a strict saddle function, recent work \cite{lee2016gradient} showed that gradient descent with random initialization almost surely escapes saddle points and converges to a local minimizer. Thus, Proposition~\ref{thm:landscape} ensures that the algorithm in \eqref{eq:gradient-descent-linear} almost surely converges to a global solution of \eqref{eq:objective-linear}.


However, because there are infinite many global solutions for the overparameterized model \eqref{eq:objective-linear} and not all global solutions are of equal quality,\footnote{In other words, not all global solutions recover the underlying $\mb \theta_\star$ and $\s_\star$} convergence to a global solution alone is not sufficient for us to establish the correctness of our method. Nonetheless, as we will show in the following, the particular choice of the algorithm in \eqref{eq:gradient-descent-linear} enables it to converge to a particular \emph{regularized} global solution. 




\paragraph{Implicit sparse regularization.} To understand which solution the algorithm \eqref{eq:gradient-descent-linear} converges to, we study its \emph{gradient flow} counterpart by taking the stepsize $\tau \rightarrow 0 $ in \eqref{eq:gradient-descent-linear}. 
Thus, the dynamics of such a gradient flow is governed by the following differential equations
\begin{equation}\label{eq:gradient-flow-linear}
\begin{split}
    \dot{\mb \theta}_t(\gamma, \alpha) &= - \mb J^\top \r_t(\gamma, \alpha), \\
    \dot{\u}_t(\gamma, \alpha) &= - 2 \alpha \cdot \u_t(\gamma, \alpha) \odot \r_t(\gamma, \alpha), \\
    \dot{\v}_t(\gamma, \alpha) &= 2 \alpha \cdot \v_t(\gamma, \alpha) \odot \r_t(\gamma, \alpha),
\end{split}
\end{equation}
where we define
\begin{equation}\label{eq:def-residual}
    \r_t(\gamma, \alpha) = \J \mb \theta_t(\gamma, \alpha) + \u_t(\gamma, \alpha) \odot \u_t(\gamma, \alpha) - \v_t(\gamma, \alpha) \odot \v_t(\gamma, \alpha) - \y.
\end{equation}
Here, we assume that $\mb \theta, \mb u$, and $\mb v$ are initialized at
\begin{equation}\label{eq:gradient-flow-init-linear}
   \mb \theta_0(\gamma, \alpha) = \mb 0, ~\u_0(\gamma, \alpha) = \gamma \1, ~\v_0(\gamma, \alpha) = \gamma \1,
\end{equation}
with some small $\gamma >0$. 
Solving the differential equations in \eqref{eq:gradient-flow-linear} gives the gradient flow
\begin{equation}\label{eq:gradient-flow-equation}
   \mb \theta_t(\gamma, \alpha) = \mb J^\top \mb \nu_t(\gamma, \alpha), \quad
    \u_t(\gamma, \alpha) = \gamma \exp(2\alpha \mb \nu_t(\gamma, \alpha)), \quad
    \v_t(\gamma, \alpha) = \gamma \exp(-2\alpha \mb \nu_t(\gamma, \alpha)),
\end{equation}
where we define
\begin{equation}\label{eq:def-nu}
    \mb \nu_t(\gamma, \alpha) \doteq - \int_0^t \r_\tau(\gamma, \alpha) d\tau.
\end{equation} 
The following result shows that the solution that the gradient flow $\big(\mb \theta_t(\gamma, \alpha), \u_t(\gamma, \alpha), \v_t(\gamma, \alpha)\big)$ in \eqref{eq:gradient-flow-equation} converges to at $t \to \infty$ is a global solution to \eqref{eq:objective-linear} that is \emph{regularized} with a particular of $(\gamma, \alpha)$.




\begin{proposition}\label{thm:implicit-bias}
Consider the gradient flow in \eqref{eq:gradient-flow-equation} with the initialization in \eqref{eq:gradient-flow-init-linear}. 
\begin{itemize}[leftmargin=*]
    \item \textbf{(Global convergence)} For any $(\gamma, \alpha)$, if the limit 
    \begin{equation}\label{eq:def-gradient-flow-infty}
        \Big(\mb \theta_\infty(\gamma, \alpha), \,\u_\infty(\gamma, \alpha), \,\v_\infty(\gamma, \alpha)\Big) 
        \doteq 
        \lim_{t \to \infty} \Big(\mb \theta_t(\gamma, \alpha), \,\u_t(\gamma, \alpha),\, \v_t(\gamma, \alpha)\Big) 
    \end{equation}
    of the gradient flow exists, then $\big(\mb \theta_\infty(\gamma, \alpha), \u_\infty(\gamma, \alpha), \v_\infty(\gamma, \alpha)\big)$ is a global solution to \eqref{eq:objective-linear}. 
    \item \textbf{(Implicit regularization)} Fix any $\lambda > 0$ and let $\alpha$ be a function of $\gamma$ as
    \begin{equation}\label{eq:implicit-bias-parameter-relation}
        \alpha(\gamma) = - \frac{\log \gamma }{2\lambda}.
    \end{equation}
    If the limit
    \begin{equation}\label{eq:def-theta-hat}
        \Big( \wh{\mb \theta}, \, \wh{\u}, \, \wh{\v} \Big) 
        \doteq \lim_{\gamma \to 0}  \Big( \mb \theta_\infty(\gamma, \alpha(\gamma)), \, \u_\infty(\gamma, \alpha(\gamma)),\, \v_\infty(\gamma, \alpha(\gamma)) \Big)
    \end{equation}
    exists, then $\big( \wh{\mb \theta}, \, \wh{\u}, \, \wh{\v} \big)$ is a global solution to \eqref{eq:objective-linear}. 
    In particular, let
    \begin{equation}\label{eq:def-s-hat}
        \wh{\s} \doteq \wh{\u} \odot \wh{\u} - \wh{\v} \odot \wh{\v},
    \end{equation}
    then $(\wh{\mb \theta}, \wh{\s})$ is an optimal solution to the following convex program 
    \begin{align}\label{eq:convex-linear}
        \min_{\mb \theta, \, \s} \frac{1}{2} \norm{\mb \theta}{2}^2 + \lambda \norm{\s}{1},\quad \text{s.t.} \quad \mb y = \mb J \mb \theta + \s.
    \end{align}

\end{itemize}

\end{proposition}

As we observe from the above result, because $(\wh{\mb \theta}, \wh{\s})$ that the gradient flow \eqref{eq:gradient-flow-linear} converges to is also an optimal solution of \eqref{eq:convex-linear}, it implies that $(\wh{\mb \theta}, \wh{\s})$ is regularized. In particular, the $\ell_1$-norm regularization on $\s$ comes as a result of implicit regularization on overparameterization $\mb s = \u \odot \u - \v \odot \v$, leading to a sparse solution on $\mb s$ as we desired. On the other hand, the $\ell_2$ regularization on $\mb \theta$ leads to the desired  minimum $\ell_2$-norm solution as we discussed in \eqref{eq:property-of-theta}. Thus, the only question remains is whether the ground truth $(\mb \theta_\star, \mb s_\star)$ in \eqref{eq:model-linear} can be identified through solving the convex problem \eqref{eq:convex-linear}, which we will discuss in the following \Cref{sec:exact-recovery}.



\begin{figure}[t]
\centering  
\ifdefined\ISARXIV
\includegraphics[width=0.75\linewidth,trim={1cm 0 3cm 0},clip]{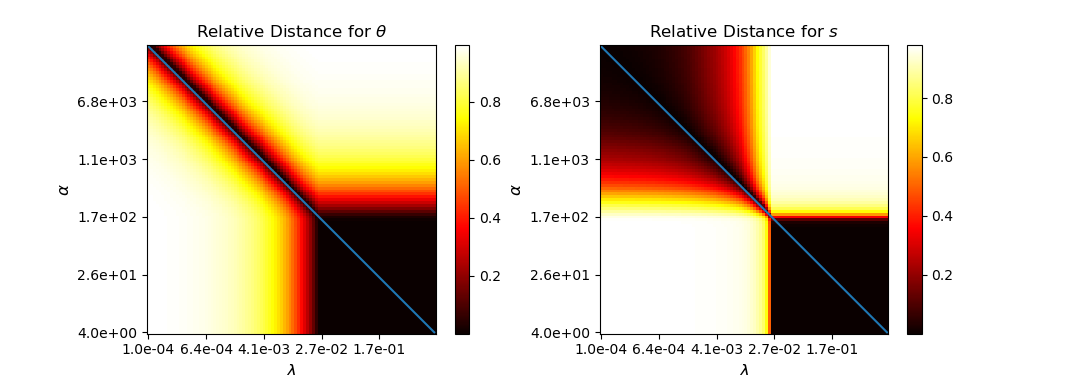}
\else
\includegraphics[width=0.99\linewidth,trim={1cm 0 3cm 0},clip]{figs/implicit_bias.png}
\fi
\caption{The gradient descent in \eqref{eq:gradient-descent-linear} and the convex problem in \eqref{eq:convex-linear} produce the same solutions with $\alpha = - \frac{\log \gamma}{2\lambda}$. 
For fixed data $(\J, \y)$, left figure shows the relative difference $\frac{\|\mb \theta_\alpha - \mb \theta_\lambda\|_2}{\max\{\|\mb \theta_\alpha\|_2, \|\mb \theta_\lambda\|_2\}}$ between the solution $\mb \theta_\alpha$ to \eqref{eq:gradient-descent-linear} with varying values of $\alpha$ (in y-axis) and the solution $\mb \theta_\lambda$ computed from \eqref{eq:convex-linear} with varying values of $\lambda$ (in x-axis). 
Likewise, right figure shows the relative difference for $\s$. 
Blue line shows the curve $\alpha = - \frac{\log \gamma}{2\lambda}$ where $\gamma$ is fixed to $\exp{(-8)}$ in all experiments. }
\label{fig:implicit_bias}

\end{figure}


\paragraph{Numerical verification.}
While Proposition~\ref{thm:implicit-bias} is proved for gradient flow with both learning rate $\tau \rightarrow 0 $ and initialization scale $\gamma \rightarrow 0$, we numerically show that such a result also holds non-asymptotically with finitely small $\tau$ and $\gamma$.


Given a tuple $(N, p, r, k)$ of model parameters, we generate simulation data $(\J, \mb \theta_\star, \s_\star, \y)$ as follows. 
The matrix $\J \in \bb R^{N \times p}$ is generated by multiplying two randomly generated matrices of shape $N\times r$ and $r \times p$, respectively with entries drawn i.i.d. from a standard Gaussian distribution. 
The sparse vector $\s_\star \in \bb R^N$ is generated by randomly choosing $k$ entries to be i.i.d. standard Gaussian, with the rest of entries zero. Then, we generate a vector $\mb \theta \in \bb R^p$ with all entries drawn i.i.d. from a standard Gaussian distribution, and let $\y = \J \mb \theta + \s_\star$.  Finally, we set $\mb \theta_\star$ as the minimum $\ell_2$-norm solution according to \eqref{eq:property-of-theta}. 

In this experiment, we choose and fix $(N, p, r, k) = (20, 40, 3, 3)$ for the data generation described above.
With a varying learning rate $\alpha \in [4, 4000]$, we compute $(\mb \theta_\alpha, \s_\alpha)$ as the solution provided by gradient descent in \eqref{eq:gradient-descent-linear} with an initialization by \eqref{eq:gradient-flow-init-linear} with $\gamma = e^{-8}$.
With a varying regularization $\lambda \in [0.0001, 1]$ in \eqref{eq:convex-linear}, we compute $(\mb \theta_\lambda, \s_\lambda)$ as the solution provided by the convex problem in \eqref{eq:convex-linear} with weight parameter $\lambda$.\footnote{We use the ECOS solver \cite{domahidi2013ecos} provided in CVXPY \cite{diamond2016cvxpy} for solving \eqref{eq:convex-linear}.}
\Cref{fig:implicit_bias} provides a visualization of the relative difference $ \rho = \frac{\|\mb \theta_\alpha - \mb \theta_\lambda\|_2}{\max\{\|\mb \theta_\alpha\|_2, \|\mb \theta_\lambda\|_2\}}$ between $\mb \theta_\alpha$ and $\mb \theta_\lambda$ (and likewise for $\s$), across all pairs of $(\alpha, \lambda)$.
We can observe that as long as $(\alpha, \lambda)$ satisfies the relationship in \eqref{eq:implicit-bias-parameter-relation}, the relative difference $\rho$ is small for $\mb \theta$, which is also true for $\s$. On the other hand, the relative differences can be large if \eqref{eq:implicit-bias-parameter-relation} is not satisfied, corroborating Proposition \ref{thm:implicit-bias}.

\subsection{Exact Recovery under Incoherence Conditions}
\label{sec:exact-recovery}


Given the overparameterized model \eqref{eq:model-linear} with $\mb y \in \bb R^N$, $\mb \theta \in \bb R^p$, and $p \gg N$, there is no enough information from $\mb y$ to recover $\mb \theta_\star$ and $\s_\star$ even with the prior information that $\s_\star$ is sparse --  any given vector $\y \in \bb R^p$ can be decomposed as a summation of an arbitrary sparse vector $\mb s$ and a vector $\mb \theta$ cooked up from the column space of $\J$ \emph{as long as $\J$ has full row rank}. 

For the solution $\mb \theta_\star$ and $\s_\star$ to be identified, first, we assume that $\mb J$ is \emph{low-rank}, where it has been empirically observed in practical deep neural network $f_{\mb \theta}$ that the Jacobian matrix $\mb J$ of $f_{\mb \theta}$ is approximately low-rank \cite{oymak2019generalization}.\footnote{Our low-rank assumption is an idealization of the \emph{approximate} low-rank property of the Jacobian, which simplifies our analysis but at the cost that our model is not able to overfit to any corrupted labels as a deep neural network. We leave the study under \emph{approximate} low-rank assumption to future work.} However, the low-rank condition of $\mb J$ alone does \emph{not} guarantee identifiability, because it cannot address the separability between $\mb J\mb \theta_\star$ and $\mb s_\star$ -- following a similar argument as that in \cite{candes2011robust}, if any column of $\J$ has a single nonzero entry, then any $\s_\star$ that is supported on the same entry cannot be recovered without ambiguities. Hence, we further assume that the column space of $\J$ and the standard basis $[\e_1, \ldots, \e_N] \doteq \diag\{1, \ldots, 1\} \in \RR^{N \times N}$ are \emph{incoherent}, defined as follows.

\begin{definition}[\cite{candes2011robust}]\label{def:incoherence}
Let $\J = \U \mb \Sigma \V^\top \in \RR^{N \times p}$ be the compact SVD of $\J$ and $r$ be the rank of $\J$. The coherence of $\J$ (w.r.t. the standard basis) is defined as 
\begin{equation}
    \mu(\J) = \frac{N}{r} \max_{1 \le i \le N} \|\U^\top \e_i\|_2^2.
\end{equation}
\end{definition}

It should be noted that the low-rank and incoherence assumptions are common for matrix recovery \cite{davenport2016overview,chi2019nonconvex}. Based upon the above assumptions on $\mb J$ and $\mb s_\star$, we show the following. 
\begin{proposition}\label{thm:exact-recovery}
Let $r$ be the rank of $\J$ and $k$ be the number of nonzero entries of $\s_\star$.
If we have 
\begin{equation}\label{eq:incoherence-condition}
    k^2 r < \frac{N}{4\mu(\J)},
\end{equation}
then the solution to \eqref{eq:convex-linear} is $(\mb \theta_\star, \s_\star)$  for any $\lambda > \lambda_0$, where $\lambda_0 > 0$ is a scalar depending on  $(\J, \mb \theta_\star, k)$. 
\end{proposition}

Thus, combining this result with Proposition~\ref{thm:implicit-bias}, the gradient flow in \eqref{eq:gradient-flow-linear} with initialization \eqref{eq:gradient-flow-init-linear} converges to $(\mb \theta_\star, \s_\star)$ when the choice of learning rate ratio $\alpha$ in \eqref{eq:gradient-flow-linear} is smaller than a certain threshold, justifying our claim in \Cref{thm:main}.

\ifdefined\ISARXIV
\begin{figure}[t]
\centering  
\subfigure[Varying $k$ with fixed $r=20$. \label{fig:effect-of-lambda-varying-k}]{\includegraphics[width=0.7\linewidth,trim={0 0.9cm 0 0},clip]{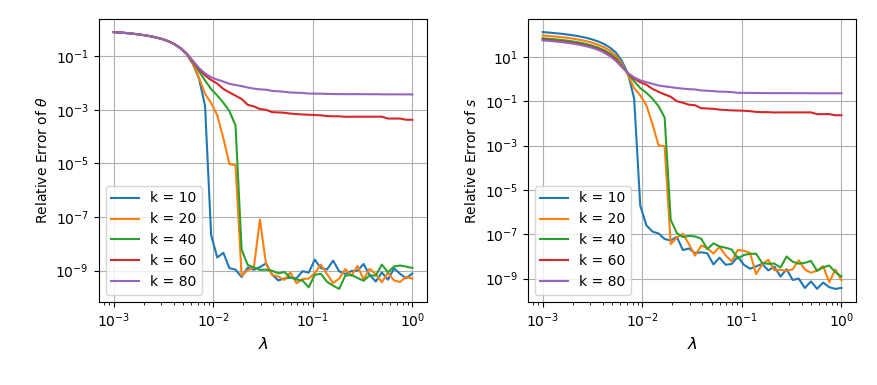}}
~
\subfigure[Varying $r$ with fixed $k=20$. \label{fig:effect-of-lambda-varying-r}]{\includegraphics[width=0.7\linewidth,trim={0 0.5cm 0 0},clip]{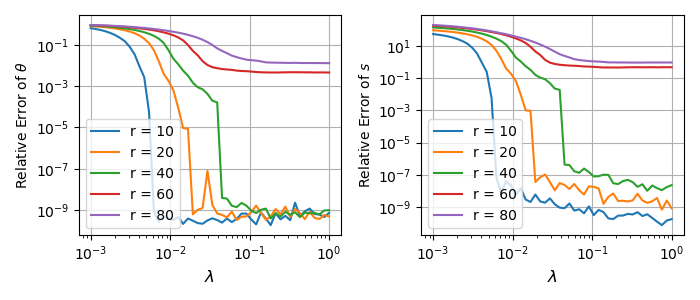}}
\caption{Effect of model parameter $\lambda$ for exact recovery by \eqref{eq:convex-linear}. 
The $y$-axis is the relative error of $\mb \theta$ (left) and $\s$ (right) defined as $\frac{\|\mb\theta - \mb\theta_\star\|_2}{\|\mb \theta_\star\|_2}$ and $\frac{\|\s - \s_\star\|_2}{\|\s_\star\|_2}$, respectively, where $(\mb \theta, \s)$ is the solution to \eqref{eq:convex-linear}. 
The curves are averages over $10$ independent trials. }
\label{fig:effect-of-lambda}
\end{figure}
\else
\begin{figure}[t]
\centering  
\subfigure[Varying $k$ with fixed $r=20$. \label{fig:effect-of-lambda-varying-k}]{\includegraphics[width=0.90\linewidth,trim={0 0.9cm 0 0},clip]{figs/effect_of_alpha_varying_k.png}}
\\
\subfigure[Varying $r$ with fixed $k=20$. \label{fig:effect-of-lambda-varying-r}]{\includegraphics[width=0.88\linewidth,trim={0 0.5cm 0 0},clip]{figs/effect_of_alpha_varying_r.png}}
\caption{Effect of model parameter $\lambda$ for exact recovery by \eqref{eq:convex-linear}. 
The $y$-axis is the relative error of $\mb \theta$ (left) and $\s$ (right) defined as $\frac{\|\mb\theta - \mb\theta_\star\|_2}{\|\mb \theta_\star\|_2}$ and $\frac{\|\s - \s_\star\|_2}{\|\s_\star\|_2}$, respectively, where $(\mb \theta, \s)$ is the solution to \eqref{eq:convex-linear}. 
The curves are averages over $10$ independent trials. }
\label{fig:effect-of-lambda}
\end{figure}
\fi

\paragraph{Numerical verification.} To corroborate Proposition \ref{thm:exact-recovery}, we numerically solve \eqref{eq:convex-linear} under varying conditions of $\lambda$, $r$, and $k$. The simulated data $(\J, \mb \theta_\star, \s_\star, \y)$ is generated the same way as the experimental part in \Cref{sec:landscape} with $N= 100$ and $p = 150$, and for an obtained solution $(\mb \theta,\mb s)$ via solving \eqref{eq:convex-linear}, we measure the relative recovery error $\epsilon_{\mb \theta} = \frac{\|\mb\theta - \mb\theta_\star\|_2}{\|\mb \theta_\star\|_2}$ and $\epsilon_{\mb s} = \frac{\|\s - \s_\star\|_2}{\|\s_\star\|_2}$.
\ifdefined\ISARXIV
\begin{itemize}[topsep=0.3em,itemsep=0.1em]
\else
\begin{itemize}[leftmargin=*,topsep=0.0em,noitemsep]
\fi
    \item \emph{Effects of the parameter $\lambda$.} Here, we consider the recovery with varying $\lambda \in [0.0001, 1]$. First, we fix $r = 20$ and vary $k \in \{10, 20, 40, 60, 80\}$, showing the relative recovery errors $\epsilon_{\mb \theta}$ and $\epsilon_{\mb s}$ in \Cref{fig:effect-of-lambda-varying-k}. Second, we fix $k = 20$ and vary $r \in \{10, 20, 40, 60, 80\}$, showing the results in \Cref{fig:effect-of-lambda-varying-r}. The results show a clear phase transition that correct recovery is obtained only when $\lambda$ is greater than a particular threshold $\lambda_0$. Moreover, $\lambda_0$ varies depending on $k$ and $r$, consistent with Proposition \ref{thm:exact-recovery}.
    \item \emph{Relationships between the rank $r$ and sparsity $k$.} Here, we fix $\lambda =0.1$ and plot the phase transition with respect to $r$ and $k$. For each $(r,k)$, the simulation is repeated for $20$ random instances, and for each instance we declare the recovery to be successful if $\epsilon_{\mb \theta}<0.001$ and $\epsilon_{\mb s}<0.001$. As shown in \Cref{fig:phase-transition}, the phase transition is consistent with Proposition \ref{thm:exact-recovery} that  successful recovery is achieved only when both $k$ and $r$ are small.
\end{itemize}

\begin{figure}[t]
\centering  
\ifdefined\ISARXIV
\includegraphics[width=0.7\linewidth,trim={0 0.2cm 0 0},clip]{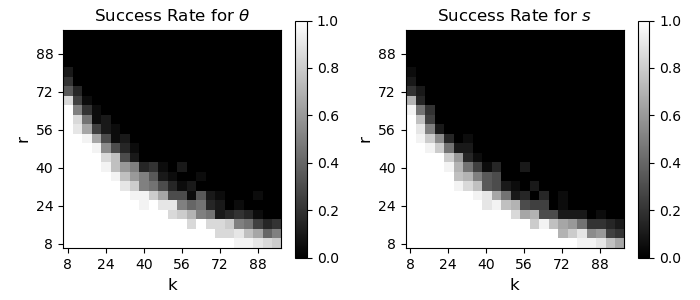}
\else
\includegraphics[width=0.96\linewidth,trim={0 0.2cm 0 0},clip]{figs/phase_transition.png}
\fi

\caption{Phase transition for solving \eqref{eq:convex-linear} over $20$ trials, with fixed $\lambda = 0.1$ and varying $k$, $r$. Recovery is declared success if $\frac{\|\mb\theta - \mb\theta_\star\|_2}{\|\mb \theta_\star\|_2}<0.001$ (left) and $\frac{\|\s - \s_\star\|_2}{\|\s_\star\|_2}<0.001$ (right).}
\label{fig:phase-transition}
\end{figure}

\section{Related Work and Discussion}\label{sec:conclusion}
\subsection{Prior Arts on Implicit Regularization}
Since overparameterized deep neural networks do not overfit (in the absence of data corruption) even without any explicit regularization \cite{zhang2021understanding}, it is argued that there are implicit regularizations pf learning algorithms that enable the models to converge to desired solutions. Under the assumption of linear or deep linear models,
many work characterized the mathematics of such implicit bias via explicit regularizations \cite{soudry2018implicit,gunasekar2018implicit,li2018algorithmic,oymak2019overparameterized,arora2019implicit,razin2020implicit,li2020towards,ji2020gradient,stoger2021small,jacot2021deep}. 
Among those, the closest related to ours include \cite{vaskevicius2019implicit,zhao2019implicit,woodworth2020kernel,li2021implicit,chou2021more}, which studied the implicit \emph{sparse} regularization induced by a term of the form $\u \odot \u - \v \odot \v$.


While all the above works aim to understand implicit regularization by studying linear models, the practical benefits of such studies are unclear. Our work provides an inspiring result showing that principled design with implicit regularization leads to robust learning of over-parameterized models.
In particular, our model in \eqref{eq:main_model} is motivated by existing studies on the implicit sparse regularization, but adds such a regularization to an (already) implicitly regularized model for handling sparse corruptions.
In other words, two forms of implicit regularization are involved in our model which poses new problems in the design of the optimization algorithm and in mathematical analysis. To the best of our knowledge, the only prior works that use implicit sparse regularization for robust learning are \cite{you2020robust,ma2021implicit,ding2021rank} which studied the robust recovery of low-rank matrices and images.
Among them, our work extends \cite{you2020robust} to the problem of image classification with label noise, demonstrates its effectiveness, and provides dedicated theoretical analyses. Additionally, methods in \cite{ma2021implicit,ding2021rank} require a particular learning rate schedule that may not be compatible with commonly used schedules such as cosine annealing \cite{loshchilov2016sgdr} in image classification.

\subsection{Relationship to Existing Work on Label Noise}
Deep neural networks are over-parameterized hence prone to \emph{overfitting} to the label noises.
While many popular regularization techniques for alleviating overfitting, such as label smoothing \cite{szegedy2016rethinking,lukasik2020does,wei2021understanding} and \emph{mixup} \cite{zhang2018mixup}, are useful for mitigating the impact of label noise, they do not completely solve the problem due to a lack of precise noise modeling. In the following, we discuss three of the most popular line of work dedicated to the label noise problem; we refer the reader to the survey papers \cite{algan2021image,song2020learning,wei2021learning} for a comprehensive review.

\paragraph{Loss design.} 
Robust loss function, such as the $\ell_1$ loss \cite{ghosh2017robust}, is one of the most popular approaches to the label noise problem which has many recent extensions \cite{zhang2018generalized,wang2019symmetric,amid2019robust,ma2020normalized,yu2020learning,wei2021optimizing,ma2022blessing}. 
The method is based on reducing the loss associated with large outlying entries, hence the impact of label noise.
A similar idea is also explored in gradient clipping \cite{menon2019can} and loss reweighting \cite{liu2015classification,wang2017multiclass,chang2017active,zhang2021dualgraph,zetterqvist2021robust} methods.
While robust loss enables the model to learn faster from correct labels, \emph{the global solution still overfits to corrupted labels with over-parameterized models}.

\paragraph{Label transition probability.} Another popular line of work for label noise is based on the assumption that the noisy label is drawn from a probability distribution conditioned on the true label.
Here, the main task is to estimate the underlying transition probabilities.
The early work \cite{chen2015webly,goldberger2016training} encapsulates the transition probabilities as a noise adaptation layer that is stacked on top of a classification network and trained jointly in an end-to-end fashion. 
Recent work \cite{patrini2017making} uses separated procedures to estimate the transition probabilities, the success of which requires either the availability of a clean validation data \cite{hendrycks2018using} or additional data assumptions \cite{xia2019anchor,zhu2021clusterability,li2021provably,zhang2021learning}. 
Even if the underlying transition probabilities can be correctly recovered, \emph{overfitting is only prevented asymptotically, requiring sufficiently many samples of corrupted labels for each input \cite{patrini2017making}, which is not practical.}

\paragraph{Label correction.}
In contrast to the above methods, our method completely avoids overfitting even with finite training samples. 
This is achieved by the over-parameterization term $\u \odot \u - \v \odot \v$ in \eqref{eq:main_model} which recovers the clean labels. 
Hence, our method is related to techniques based on noisy label detection and refurbishment. 
Nonetheless, existing techniques are based on heuristic argument about different behaviors of clean and corrupted samples in the training process, such as properties of learned representations \cite{kim2021fine,ma2018dimensionality,jiang2020beyond}, prediction consistency \cite{reed2014training,song2019selfie}, learning speed \cite{li2020gradient,liu2020early,Liu2021AdaptiveEC}, margin \cite{lin2021learning}, confidence \cite{cheng2021learning}.
They often need to be combined with engineering tricks such as moving average \cite{huang2020self,liu2020early} and burning-in \cite{zheng2020error} to make them work well.
Finally, the work \cite{hu2019simple} introduces a variable to estimate the label noise in a way similar to \eqref{eq:main_model}. 
However, the variable is not over-parameterized to induce sparsity, and their method does not have competitive performance.

\subsection{Sparsity in Deep Learning} 

Our method is broadly related to existing efforts on introducing sparsity into deep learning \cite{hoefler2021sparsity}, but is notably different in both the objective of introducing sparsity, the origin of sparsity, and how sparsity is enforced.
First, previous exploration of sparsity primarily aims to improve training and inference efficiency with large-scale models, while our paper focuses on robust training under label noise.  
Second, previous introduction of sparsity is often motivated by its presence in biological brains, but there is still a lack of clean understanding of how sparsity helps with learning. 
In contrast,  sparsity in our method has the clear mathematically meaning that the percentage of corrupted labels is small. 
Finally, while pruning \cite{liu2021we,chen2021sparsity} is a dominant approach for obtaining sparsity, our method leverages the implicit bias of gradient descent associated with a particular sparse over-parameterization.

\subsection{Limitations and Future Directions}

\paragraph{Choice of optimization algorithms. } 
Our SOP method is based on introducing more parameters to an already over-parameterized model, hence relies critically on the choice of the optimization algorithm to induce the desired implicit regularization.
For \emph{vanilla} gradient descent, our analysis in \Cref{sec:theory} shows that it has the desired implicit regularization by design.
In practical deep network training, it is more common to use the \emph{stochastic} gradient descent \emph{with momentum}. 
While not theoretically justified, experiments in \Cref{sec:experiments} show that our method works with such practical variants.
This may not come as a surprise, because existing studies already show that stochastic gradient descent \cite{nacson2019stochastic} and momentum acceleration \cite{wang2021momentum} have the same implicit bias as the vanilla gradient decent under certain models. 
We leave the extension of such results to our method as future works.

\paragraph{Modeling of label noise. }
Our method is based on the assumption that the label noise matrix $\mb S_\star = [\s_{\star 1}, \ldots, \s_{\star N}]$, where $\s_{\star i}$ is the difference between the observed label $\y_i$ and the underlying true label, is a \emph{sparse} matrix. 
We made no additional assumption on the sparsity pattern of $\mb S_\star$, other than the non-negative and non-positive constraints discussed in \eqref{eq:image_classification_model}.
In practice, it is usually the case that certain pairs of classes are more similar hence more easily confusing with each other than other pairs.
As a result, certain blocks of $\mb S_\star$ tend to have more non-zero entries than the others. 
When there is a prior on which blocks may have more non-zero entries, our method may be adapted by using a \emph{weighted} sparse regularization for the corresponding blocks. 
When there is no such prior, our method may be adapted by using a \emph{group} sparse regularization \cite{neyshabur2014search,tibshirani2021equivalences}.

\paragraph{Robustness of learned representations against label noise.} Recently, a line of work showed an intriguing and universal phenomenon of learned deep representations under natural setting \cite{papyan2020prevalence,han2021neural,zhu2021geometric,fang2021exploring}, that the last-layer representations of each class collapse to a single dimension. However, the collapsed representation loses the variability of the data and is vulnerable to corruptions such as label noise. Another line of recent work \cite{yu2020learning,chan2021redunet} empirically showed and argued that mapping each class to a linearly separable subspace with maximum dimension (instead of collapsing them to a vertex of Simplex ETF) can improve robustness against label noise and random data corruptions. Based upon the proposed implicit regularizations of the network, it would be interesting to study and further justify the robustness of the proposed methods in terms of the learned last-layer representations.

\section*{Acknowledgement}

SL and QQ were partially supported by NSF grant DMS 2009752. 
SL was partially supported by NSF NRT-HDR Award 1922658 and Alzheimer’s Association grant AARG-NTF-21-848627.
QQ also acknowledge support from NSF CAREER 2143904, NSF CCF 2212066, and ONR N00014-22-1-2529.
ZZ acknowledges support from NSF grants CCF 2008460 and CCF 2106881. 
Part of this work was done when CY was at University of California, Berkeley and was supported by Tsinghua-Berkeley Shenzhen Institute Research Fund. 
The authors acknowledge helpful discussion with Ryan Chan from Johns Hopkins University. 

\newpage 

{\small
\medskip
\bibliographystyle{ieeetr}
\bibliography{icml2022/optimization}

\begin{thebibliography}{100}

\bibitem{belkin2019reconciling}
M.~Belkin, D.~Hsu, S.~Ma, and S.~Mandal, ``Reconciling modern machine-learning
  practice and the classical bias--variance trade-off,'' {\em Proceedings of
  the National Academy of Sciences}, vol.~116, no.~32, pp.~15849--15854, 2019.

\bibitem{yang2020rethinking}
Z.~Yang, Y.~Yu, C.~You, J.~Steinhardt, and Y.~Ma, ``Rethinking bias-variance
  trade-off for generalization of neural networks,'' in {\em International
  Conference on Machine Learning}, pp.~10767--10777, PMLR, 2020.

\bibitem{neyshabur2014search}
B.~Neyshabur, R.~Tomioka, and N.~Srebro, ``In search of the real inductive
  bias: On the role of implicit regularization in deep learning,'' {\em arXiv
  preprint arXiv:1412.6614}, 2014.

\bibitem{song2020learning}
H.~Song, M.~Kim, D.~Park, Y.~Shin, and J.-G. Lee, ``Learning from noisy labels
  with deep neural networks: A survey,'' {\em arXiv preprint arXiv:2007.08199},
  2020.

\bibitem{zhang2021understanding}
C.~Zhang, S.~Bengio, M.~Hardt, B.~Recht, and O.~Vinyals, ``Understanding deep
  learning (still) requires rethinking generalization,'' {\em Communications of
  the ACM}, vol.~64, no.~3, pp.~107--115, 2021.

\bibitem{frenay2013classification}
B.~Fr{\'e}nay and M.~Verleysen, ``Classification in the presence of label
  noise: a survey,'' {\em IEEE transactions on neural networks and learning
  systems}, vol.~25, no.~5, pp.~845--869, 2013.

\bibitem{candes2005decoding}
E.~J. Candes and T.~Tao, ``Decoding by linear programming,'' {\em IEEE
  transactions on information theory}, vol.~51, no.~12, pp.~4203--4215, 2005.

\bibitem{candes2011robust}
E.~J. Cand{\`e}s, X.~Li, Y.~Ma, and J.~Wright, ``Robust principal component
  analysis?,'' {\em Journal of the ACM (JACM)}, vol.~58, no.~3, pp.~1--37,
  2011.

\bibitem{wright2008robust}
J.~Wright, A.~Y. Yang, A.~Ganesh, S.~S. Sastry, and Y.~Ma, ``Robust face
  recognition via sparse representation,'' {\em IEEE transactions on pattern
  analysis and machine intelligence}, vol.~31, no.~2, pp.~210--227, 2008.

\bibitem{claerbout1973robust}
J.~F. Claerbout and F.~Muir, ``Robust modeling with erratic data,'' {\em
  Geophysics}, vol.~38, no.~5, pp.~826--844, 1973.

\bibitem{vaskevicius2019implicit}
T.~Vaskevicius, V.~Kanade, and P.~Rebeschini, ``Implicit regularization for
  optimal sparse recovery,'' in {\em Advances in Neural Information Processing
  Systems}, pp.~2968--2979, 2019.

\bibitem{zhao2019implicit}
P.~Zhao, Y.~Yang, and Q.-C. He, ``Implicit regularization via hadamard product
  over-parametrization in high-dimensional linear regression,'' {\em arXiv
  preprint arXiv:1903.09367}, 2019.

\bibitem{you2020robust}
C.~You, Z.~Zhu, Q.~Qu, and Y.~Ma, ``Robust recovery via implicit bias of
  discrepant learning rates for double over-parameterization,'' {\em Advances
  in Neural Information Processing Systems}, vol.~33, pp.~17733--17744, 2020.

\bibitem{jacot2018neural}
A.~Jacot, F.~Gabriel, and C.~Hongler, ``Neural tangent kernel: Convergence and
  generalization in neural networks,'' {\em arXiv preprint arXiv:1806.07572},
  2018.

\bibitem{chizat2018lazy}
L.~Chizat, E.~Oyallon, and F.~Bach, ``On lazy training in differentiable
  programming,'' {\em Advances in Neural Information Processing Systems},
  vol.~32, pp.~2937--2947, 2019.

\bibitem{krizhevsky2012imagenet}
A.~Krizhevsky, I.~Sutskever, and G.~E. Hinton, ``Imagenet classification with
  deep convolutional neural networks,'' {\em Advances in neural information
  processing systems}, vol.~25, pp.~1097--1105, 2012.

\bibitem{krizhevsky2009learning}
A.~Krizhevsky, G.~Hinton, {\em et~al.}, ``Learning multiple layers of features
  from tiny images,'' 2009.

\bibitem{NEURIPS2018_a19744e2}
B.~Han, Q.~Yao, X.~Yu, G.~Niu, M.~Xu, W.~Hu, I.~Tsang, and M.~Sugiyama,
  ``Co-teaching: Robust training of deep neural networks with extremely noisy
  labels,'' in {\em Advances in Neural Information Processing Systems}
  (S.~Bengio, H.~Wallach, H.~Larochelle, K.~Grauman, N.~Cesa-Bianchi, and
  R.~Garnett, eds.), vol.~31, Curran Associates, Inc., 2018.

\bibitem{liu2020early}
S.~Liu, J.~Niles-Weed, N.~Razavian, and C.~Fernandez-Granda, ``Early-learning
  regularization prevents memorization of noisy labels,'' {\em Advances in
  Neural Information Processing Systems}, vol.~33, 2020.

\bibitem{xia2020robust}
X.~Xia, T.~Liu, B.~Han, C.~Gong, N.~Wang, Z.~Ge, and Y.~Chang, ``Robust
  early-learning: Hindering the memorization of noisy labels,'' in {\em
  International Conference on Learning Representations}, 2020.

\bibitem{wei2021learning}
J.~Wei, Z.~Zhu, H.~Cheng, T.~Liu, G.~Niu, and Y.~Liu, ``Learning with noisy
  labels revisited: A study using real-world human annotations,'' {\em arXiv
  preprint arXiv:2110.12088}, 2021.

\bibitem{xiao2015learning}
T.~Xiao, T.~Xia, Y.~Yang, C.~Huang, and X.~Wang, ``Learning from massive noisy
  labeled data for image classification,'' in {\em Proceedings of the IEEE
  conference on computer vision and pattern recognition}, pp.~2691--2699, 2015.

\bibitem{li2017webvision}
W.~Li, L.~Wang, W.~Li, E.~Agustsson, and L.~Van~Gool, ``Webvision database:
  Visual learning and understanding from web data,'' {\em arXiv preprint
  arXiv:1708.02862}, 2017.

\bibitem{he2016deep}
K.~He, X.~Zhang, S.~Ren, and J.~Sun, ``Deep residual learning for image
  recognition,'' in {\em Proceedings of the IEEE conference on computer vision
  and pattern recognition}, pp.~770--778, 2016.

\bibitem{patrini2017making}
G.~Patrini, A.~Rozza, A.~Krishna~Menon, R.~Nock, and L.~Qu, ``Making deep
  neural networks robust to label noise: A loss correction approach,'' in {\em
  Proceedings of the IEEE conference on computer vision and pattern
  recognition}, pp.~1944--1952, 2017.

\bibitem{zhang2018generalized}
Z.~Zhang and M.~R. Sabuncu, ``Generalized cross entropy loss for training deep
  neural networks with noisy labels,'' in {\em 32nd Conference on Neural
  Information Processing Systems (NeurIPS)}, 2018.

\bibitem{wang2019symmetric}
Y.~Wang, X.~Ma, Z.~Chen, Y.~Luo, J.~Yi, and J.~Bailey, ``Symmetric cross
  entropy for robust learning with noisy labels,'' in {\em Proceedings of the
  IEEE/CVF International Conference on Computer Vision}, pp.~322--330, 2019.

\bibitem{li2020dividemix}
J.~Li, R.~Socher, and S.~C. Hoi, ``Dividemix: Learning with noisy labels as
  semi-supervised learning,'' {\em arXiv preprint arXiv:2002.07394}, 2020.

\bibitem{cheng2021learning}
H.~Cheng, Z.~Zhu, X.~Li, Y.~Gong, X.~Sun, and Y.~Liu, ``Learning with
  instance-dependent label noise: A sample sieve approach,'' in {\em
  International Conference on Learning Representations}, 2021.

\bibitem{kalimeris2019sgd}
D.~Kalimeris, G.~Kaplun, P.~Nakkiran, B.~Edelman, T.~Yang, B.~Barak, and
  H.~Zhang, ``Sgd on neural networks learns functions of increasing
  complexity,'' {\em Advances in Neural Information Processing Systems},
  vol.~32, pp.~3496--3506, 2019.

\bibitem{ge2015escaping}
R.~Ge, F.~Huang, C.~Jin, and Y.~Yuan, ``Escaping from saddle points—online
  stochastic gradient for tensor decomposition,'' in {\em Conference on
  learning theory}, pp.~797--842, PMLR, 2015.

\bibitem{lee2016gradient}
J.~D. Lee, M.~Simchowitz, M.~I. Jordan, and B.~Recht, ``Gradient descent only
  converges to minimizers,'' in {\em Conference on learning theory},
  pp.~1246--1257, PMLR, 2016.

\bibitem{domahidi2013ecos}
A.~Domahidi, E.~Chu, and S.~Boyd, ``{ECOS}: {A}n {SOCP} solver for embedded
  systems,'' in {\em European Control Conference (ECC)}, pp.~3071--3076, 2013.

\bibitem{diamond2016cvxpy}
S.~Diamond and S.~Boyd, ``{CVXPY}: {A} {P}ython-embedded modeling language for
  convex optimization,'' {\em Journal of Machine Learning Research}, vol.~17,
  no.~83, pp.~1--5, 2016.

\bibitem{oymak2019generalization}
S.~Oymak, Z.~Fabian, M.~Li, and M.~Soltanolkotabi, ``Generalization guarantees
  for neural networks via harnessing the low-rank structure of the jacobian,''
  {\em arXiv preprint arXiv:1906.05392}, 2019.

\bibitem{davenport2016overview}
M.~A. Davenport and J.~Romberg, ``An overview of low-rank matrix recovery from
  incomplete observations,'' {\em IEEE Journal of Selected Topics in Signal
  Processing}, vol.~10, no.~4, pp.~608--622, 2016.

\bibitem{chi2019nonconvex}
Y.~Chi, Y.~M. Lu, and Y.~Chen, ``Nonconvex optimization meets low-rank matrix
  factorization: An overview,'' {\em IEEE Transactions on Signal Processing},
  vol.~67, no.~20, pp.~5239--5269, 2019.

\bibitem{soudry2018implicit}
D.~Soudry, E.~Hoffer, M.~S. Nacson, S.~Gunasekar, and N.~Srebro, ``The implicit
  bias of gradient descent on separable data,'' {\em The Journal of Machine
  Learning Research}, vol.~19, no.~1, pp.~2822--2878, 2018.

\bibitem{gunasekar2018implicit}
S.~Gunasekar, B.~Woodworth, S.~Bhojanapalli, B.~Neyshabur, and N.~Srebro,
  ``Implicit regularization in matrix factorization,'' in {\em 2018 Information
  Theory and Applications Workshop (ITA)}, pp.~1--10, IEEE, 2018.

\bibitem{li2018algorithmic}
Y.~Li, T.~Ma, and H.~Zhang, ``Algorithmic regularization in over-parameterized
  matrix sensing and neural networks with quadratic activations,'' in {\em
  Conference On Learning Theory}, pp.~2--47, 2018.

\bibitem{oymak2019overparameterized}
S.~Oymak and M.~Soltanolkotabi, ``Overparameterized nonlinear learning:
  Gradient descent takes the shortest path?,'' in {\em International Conference
  on Machine Learning}, pp.~4951--4960, 2019.

\bibitem{arora2019implicit}
S.~Arora, N.~Cohen, W.~Hu, and Y.~Luo, ``Implicit regularization in deep matrix
  factorization,'' in {\em Advances in Neural Information Processing Systems},
  pp.~7411--7422, 2019.

\bibitem{razin2020implicit}
N.~Razin and N.~Cohen, ``Implicit regularization in deep learning may not be
  explainable by norms,'' {\em Advances in Neural Information Processing
  Systems}, vol.~33, 2020.

\bibitem{li2020towards}
Z.~Li, Y.~Luo, and K.~Lyu, ``Towards resolving the implicit bias of gradient
  descent for matrix factorization: Greedy low-rank learning,'' in {\em
  International Conference on Learning Representations}, 2020.

\bibitem{ji2020gradient}
Z.~Ji, M.~Dud{\'\i}k, R.~E. Schapire, and M.~Telgarsky, ``Gradient descent
  follows the regularization path for general losses,'' in {\em Conference on
  Learning Theory}, pp.~2109--2136, PMLR, 2020.

\bibitem{stoger2021small}
D.~St{\"o}ger and M.~Soltanolkotabi, ``Small random initialization is akin to
  spectral learning: Optimization and generalization guarantees for
  overparameterized low-rank matrix reconstruction,'' {\em Advances in Neural
  Information Processing Systems}, vol.~34, 2021.

\bibitem{jacot2021deep}
A.~Jacot, F.~Ged, F.~Gabriel, B.~Simsek, and C.~Hongler, ``Deep linear networks
  dynamics: Low-rank biases induced by initialization scale and l2
  regularization,'' {\em arXiv preprint arXiv:2106.15933}, 2021.

\bibitem{woodworth2020kernel}
B.~Woodworth, S.~Gunasekar, J.~D. Lee, E.~Moroshko, P.~Savarese, I.~Golan,
  D.~Soudry, and N.~Srebro, ``Kernel and rich regimes in overparametrized
  models,'' in {\em Conference on Learning Theory}, pp.~3635--3673, PMLR, 2020.

\bibitem{li2021implicit}
J.~Li, T.~Nguyen, C.~Hegde, and K.~W. Wong, ``Implicit sparse regularization:
  The impact of depth and early stopping,'' {\em Advances in Neural Information
  Processing Systems}, vol.~34, 2021.

\bibitem{chou2021more}
H.-H. Chou, J.~Maly, and H.~Rauhut, ``More is less: Inducing sparsity via
  overparameterization,'' {\em arXiv preprint arXiv:2112.11027}, 2021.

\bibitem{ma2021implicit}
J.~Ma and S.~Fattahi, ``Implicit regularization of sub-gradient method in
  robust matrix recovery: Don't be afraid of outliers,'' {\em arXiv preprint
  arXiv:2102.02969}, 2021.

\bibitem{ding2021rank}
L.~Ding, L.~Jiang, Y.~Chen, Q.~Qu, and Z.~Zhu, ``Rank overspecified robust
  matrix recovery: Subgradient method and exact recovery,'' {\em Advances in
  Neural Information Processing Systems}, vol.~34, 2021.

\bibitem{loshchilov2016sgdr}
I.~Loshchilov and F.~Hutter, ``Sgdr: Stochastic gradient descent with warm
  restarts,'' {\em arXiv preprint arXiv:1608.03983}, 2016.

\bibitem{szegedy2016rethinking}
C.~Szegedy, V.~Vanhoucke, S.~Ioffe, J.~Shlens, and Z.~Wojna, ``Rethinking the
  inception architecture for computer vision,'' in {\em Proceedings of the IEEE
  conference on computer vision and pattern recognition}, pp.~2818--2826, 2016.

\bibitem{lukasik2020does}
M.~Lukasik, S.~Bhojanapalli, A.~Menon, and S.~Kumar, ``Does label smoothing
  mitigate label noise?,'' in {\em International Conference on Machine
  Learning}, pp.~6448--6458, PMLR, 2020.

\bibitem{wei2021understanding}
J.~Wei, H.~Liu, T.~Liu, G.~Niu, and Y.~Liu, ``Understanding (generalized) label
  smoothing whenlearning with noisy labels,'' {\em arXiv preprint
  arXiv:2106.04149}, 2021.

\bibitem{zhang2018mixup}
H.~Zhang, M.~Cisse, Y.~N. Dauphin, and D.~Lopez-Paz, ``mixup: Beyond empirical
  risk minimization,'' in {\em International Conference on Learning
  Representations}, 2018.

\bibitem{algan2021image}
G.~Algan and I.~Ulusoy, ``Image classification with deep learning in the
  presence of noisy labels: A survey,'' {\em Knowledge-Based Systems},
  vol.~215, p.~106771, 2021.

\bibitem{ghosh2017robust}
A.~Ghosh, H.~Kumar, and P.~Sastry, ``Robust loss functions under label noise
  for deep neural networks,'' in {\em Proceedings of the AAAI Conference on
  Artificial Intelligence}, vol.~31, 2017.

\bibitem{amid2019robust}
E.~Amid, M.~K. Warmuth, R.~Anil, and T.~Koren, ``Robust bi-tempered logistic
  loss based on bregman divergences,'' {\em Advances in Neural Information
  Processing Systems}, vol.~32, pp.~15013--15022, 2019.

\bibitem{ma2020normalized}
X.~Ma, H.~Huang, Y.~Wang, S.~Romano, S.~Erfani, and J.~Bailey, ``Normalized
  loss functions for deep learning with noisy labels,'' in {\em International
  Conference on Machine Learning}, pp.~6543--6553, PMLR, 2020.

\bibitem{yu2020learning}
Y.~Yu, K.~H.~R. Chan, C.~You, C.~Song, and Y.~Ma, ``Learning diverse and
  discriminative representations via the principle of maximal coding rate
  reduction,'' {\em Advances in Neural Information Processing Systems},
  vol.~33, pp.~9422--9434, 2020.

\bibitem{wei2021optimizing}
J.~Wei and Y.~Liu, ``When optimizing $ f $-divergence is robust with label
  noise,'' in {\em International Conference on Learning Representations}, 2021.

\bibitem{ma2022blessing}
J.~Ma and S.~Fattahi, ``Blessing of nonconvexity in deep linear models: Depth
  flattens the optimization landscape around the true solution,'' {\em arXiv
  preprint arXiv:2207.07612}, 2022.

\bibitem{menon2019can}
A.~K. Menon, A.~S. Rawat, S.~J. Reddi, and S.~Kumar, ``Can gradient clipping
  mitigate label noise?,'' in {\em International Conference on Learning
  Representations}, 2019.

\bibitem{liu2015classification}
T.~Liu and D.~Tao, ``Classification with noisy labels by importance
  reweighting,'' {\em IEEE Transactions on pattern analysis and machine
  intelligence}, vol.~38, no.~3, pp.~447--461, 2015.

\bibitem{wang2017multiclass}
R.~Wang, T.~Liu, and D.~Tao, ``Multiclass learning with partially corrupted
  labels,'' {\em IEEE transactions on neural networks and learning systems},
  vol.~29, no.~6, pp.~2568--2580, 2017.

\bibitem{chang2017active}
H.-S. Chang, E.~Learned-Miller, and A.~McCallum, ``Active bias: Training more
  accurate neural networks by emphasizing high variance samples,'' {\em
  Advances in Neural Information Processing Systems}, vol.~30, pp.~1002--1012,
  2017.

\bibitem{zhang2021dualgraph}
H.~Zhang, X.~Xing, and L.~Liu, ``Dualgraph: A graph-based method for reasoning
  about label noise,'' in {\em Proceedings of the IEEE/CVF Conference on
  Computer Vision and Pattern Recognition}, pp.~9654--9663, 2021.

\bibitem{zetterqvist2021robust}
O.~Zetterqvist, R.~J{\"o}rnsten, and J.~Jonasson, ``Robust neural network
  classification via double regularization,'' {\em arXiv preprint
  arXiv:2112.08102}, 2021.

\bibitem{chen2015webly}
X.~Chen and A.~Gupta, ``Webly supervised learning of convolutional networks,''
  in {\em Proceedings of the IEEE International Conference on Computer Vision},
  pp.~1431--1439, 2015.

\bibitem{goldberger2016training}
J.~Goldberger and E.~Ben-Reuven, ``Training deep neural-networks using a noise
  adaptation layer,'' 2017.

\bibitem{hendrycks2018using}
D.~Hendrycks, M.~Mazeika, D.~Wilson, and K.~Gimpel, ``Using trusted data to
  train deep networks on labels corrupted by severe noise,'' {\em Advances in
  Neural Information Processing Systems}, vol.~31, pp.~10456--10465, 2018.

\bibitem{xia2019anchor}
X.~Xia, T.~Liu, N.~Wang, B.~Han, C.~Gong, G.~Niu, and M.~Sugiyama, ``Are anchor
  points really indispensable in label-noise learning?,'' {\em Advances in
  Neural Information Processing Systems}, vol.~32, pp.~6838--6849, 2019.

\bibitem{zhu2021clusterability}
Z.~Zhu, Y.~Song, and Y.~Liu, ``Clusterability as an alternative to anchor
  points when learning with noisy labels,'' {\em arXiv preprint
  arXiv:2102.05291}, 2021.

\bibitem{li2021provably}
X.~Li, T.~Liu, B.~Han, G.~Niu, and M.~Sugiyama, ``Provably end-to-end
  label-noise learning without anchor points,'' {\em arXiv preprint
  arXiv:2102.02400}, 2021.

\bibitem{zhang2021learning}
Y.~Zhang, G.~Niu, and M.~Sugiyama, ``Learning noise transition matrix from only
  noisy labels via total variation regularization,'' {\em arXiv preprint
  arXiv:2102.02414}, 2021.

\bibitem{kim2021fine}
T.~Kim, J.~Ko, J.~Choi, S.-Y. Yun, {\em et~al.}, ``Fine samples for learning
  with noisy labels,'' {\em Advances in Neural Information Processing Systems},
  vol.~34, 2021.

\bibitem{ma2018dimensionality}
X.~Ma, Y.~Wang, M.~E. Houle, S.~Zhou, S.~Erfani, S.~Xia, S.~Wijewickrema, and
  J.~Bailey, ``Dimensionality-driven learning with noisy labels,'' in {\em
  International Conference on Machine Learning}, pp.~3355--3364, PMLR, 2018.

\bibitem{jiang2020beyond}
L.~Jiang, D.~Huang, M.~Liu, and W.~Yang, ``Beyond synthetic noise: Deep
  learning on controlled noisy labels,'' in {\em International Conference on
  Machine Learning}, pp.~4804--4815, PMLR, 2020.

\bibitem{reed2014training}
S.~Reed, H.~Lee, D.~Anguelov, C.~Szegedy, D.~Erhan, and A.~Rabinovich,
  ``Training deep neural networks on noisy labels with bootstrapping,'' {\em
  arXiv preprint arXiv:1412.6596}, 2014.

\bibitem{song2019selfie}
H.~Song, M.~Kim, and J.-G. Lee, ``Selfie: Refurbishing unclean samples for
  robust deep learning,'' in {\em International Conference on Machine
  Learning}, pp.~5907--5915, PMLR, 2019.

\bibitem{li2020gradient}
M.~Li, M.~Soltanolkotabi, and S.~Oymak, ``Gradient descent with early stopping
  is provably robust to label noise for overparameterized neural networks,'' in
  {\em International conference on artificial intelligence and statistics},
  pp.~4313--4324, PMLR, 2020.

\bibitem{Liu2021AdaptiveEC}
S.~Liu, K.~Liu, W.~Zhu, Y.~Shen, and C.~Fernandez-Granda, ``Adaptive
  early-learning correction for segmentation from noisy annotations,'' {\em
  ArXiv}, vol.~abs/2110.03740, 2021.

\bibitem{lin2021learning}
J.~Z. Lin and J.~Bradic, ``Learning to combat noisy labels via classification
  margins,'' {\em arXiv preprint arXiv:2102.00751}, 2021.

\bibitem{huang2020self}
L.~Huang, C.~Zhang, and H.~Zhang, ``Self-adaptive training: beyond empirical
  risk minimization,'' {\em arXiv preprint arXiv:2002.10319}, 2020.

\bibitem{zheng2020error}
S.~Zheng, P.~Wu, A.~Goswami, M.~Goswami, D.~Metaxas, and C.~Chen,
  ``Error-bounded correction of noisy labels,'' in {\em International
  Conference on Machine Learning}, pp.~11447--11457, PMLR, 2020.

\bibitem{hu2019simple}
W.~Hu, Z.~Li, and D.~Yu, ``Simple and effective regularization methods for
  training on noisily labeled data with generalization guarantee,'' in {\em
  International Conference on Learning Representations}, 2019.

\bibitem{hoefler2021sparsity}
T.~Hoefler, D.~Alistarh, T.~Ben-Nun, N.~Dryden, and A.~Peste, ``Sparsity in
  deep learning: Pruning and growth for efficient inference and training in
  neural networks,'' {\em Journal of Machine Learning Research}, vol.~22,
  no.~241, pp.~1--124, 2021.

\bibitem{liu2021we}
S.~Liu, L.~Yin, D.~C. Mocanu, and M.~Pechenizkiy, ``Do we actually need dense
  over-parameterization? in-time over-parameterization in sparse training,'' in
  {\em International Conference on Machine Learning}, pp.~6989--7000, PMLR,
  2021.

\bibitem{chen2021sparsity}
T.~Chen, Z.~Zhang, S.~Balachandra, H.~Ma, Z.~Wang, Z.~Wang, {\em et~al.},
  ``Sparsity winning twice: Better robust generalization from more efficient
  training,'' in {\em International Conference on Learning Representations},
  2021.

\bibitem{nacson2019stochastic}
M.~S. Nacson, N.~Srebro, and D.~Soudry, ``Stochastic gradient descent on
  separable data: Exact convergence with a fixed learning rate,'' in {\em The
  22nd International Conference on Artificial Intelligence and Statistics},
  pp.~3051--3059, PMLR, 2019.

\bibitem{wang2021momentum}
B.~Wang, Q.~Meng, H.~Zhang, R.~Sun, W.~Chen, and Z.-M. Ma, ``Momentum doesn't
  change the implicit bias,'' {\em arXiv preprint arXiv:2110.03891}, 2021.

\bibitem{tibshirani2021equivalences}
R.~J. Tibshirani, ``Equivalences between sparse models and neural networks,''
  2021.

\bibitem{papyan2020prevalence}
V.~Papyan, X.~Han, and D.~L. Donoho, ``Prevalence of neural collapse during the
  terminal phase of deep learning training,'' {\em Proceedings of the National
  Academy of Sciences}, vol.~117, no.~40, pp.~24652--24663, 2020.

\bibitem{han2021neural}
X.~Han, V.~Papyan, and D.~L. Donoho, ``Neural collapse under mse loss:
  Proximity to and dynamics on the central path,'' {\em arXiv preprint
  arXiv:2106.02073}, 2021.

\bibitem{zhu2021geometric}
Z.~Zhu, T.~Ding, J.~Zhou, X.~Li, C.~You, J.~Sulam, and Q.~Qu, ``A geometric
  analysis of neural collapse with unconstrained features,'' {\em Advances in
  Neural Information Processing Systems}, 2021.

\bibitem{fang2021exploring}
C.~Fang, H.~He, Q.~Long, and W.~J. Su, ``Exploring deep neural networks via
  layer-peeled model: Minority collapse in imbalanced training,'' {\em
  Proceedings of the National Academy of Sciences}, vol.~118, no.~43, 2021.

\bibitem{chan2021redunet}
K.~H.~R. Chan, Y.~Yu, C.~You, H.~Qi, J.~Wright, and Y.~Ma, ``Redunet: A
  white-box deep network from the principle of maximizing rate reduction,''
  {\em arXiv preprint arXiv:2105.10446}, 2021.

\bibitem{Song2019PrestoppingHD}
H.~Song, M.~Kim, D.~Park, and J.-G. Lee, ``Prestopping: How does early stopping
  help generalization against label noise?,'' {\em ArXiv}, vol.~abs/1911.08059,
  2019.

\bibitem{mixmatch}
D.~Berthelot, N.~Carlini, I.~Goodfellow, N.~Papernot, A.~Oliver, and C.~A.
  Raffel, ``Mixmatch: A holistic approach to semi-supervised learning,'' in
  {\em Advances in Neural Information Processing Systems} (H.~Wallach,
  H.~Larochelle, A.~Beygelzimer, F.~d\textquotesingle Alch\'{e}-Buc, E.~Fox,
  and R.~Garnett, eds.), vol.~32, Curran Associates, Inc., 2019.

\bibitem{li2019learning}
J.~Li, Y.~Wong, Q.~Zhao, and M.~S. Kankanhalli, ``Learning to learn from noisy
  labeled data,'' in {\em Proceedings of the IEEE/CVF Conference on Computer
  Vision and Pattern Recognition}, pp.~5051--5059, 2019.

\bibitem{Xie2020UnsupervisedDA}
Q.~Xie, Z.~Dai, E.~H. Hovy, M.-T. Luong, and Q.~V. Le, ``Unsupervised data
  augmentation for consistency training,'' {\em arXiv: Learning}, 2020.

\bibitem{tanaka2018joint}
D.~Tanaka, D.~Ikami, T.~Yamasaki, and K.~Aizawa, ``Joint optimization framework
  for learning with noisy labels,'' in {\em Proceedings of the IEEE Conference
  on Computer Vision and Pattern Recognition}, pp.~5552--5560, 2018.

\bibitem{cohen2009compressed}
A.~Cohen, W.~Dahmen, and R.~DeVore, ``Compressed sensing and best $k$-term
  approximation,'' {\em Journal of the American mathematical society}, vol.~22,
  no.~1, pp.~211--231, 2009.

\end{thebibliography}
}

\newpage

\newpage

\appendices

\numberwithin{equation}{section}
\numberwithin{figure}{section}
\numberwithin{table}{section}

This appendix is organized as follows. 
In~\Cref{app:training_details} we provide additional details for reproducing experimental results presented in \Cref{sec:experiments}.
In~\Cref{app:proof} we provide proofs for the theoretical results presented in \Cref{sec:theory}.

\section{Training Details for Robust Classification with Label Noise}
\label{app:training_details}

\subsection{Choice of Loss Function}
\label{app:choice-of-loss}
The cross-entropy loss in \eqref{eq:ce-loss} cannot be used to optimize $\{\v_i\}$ as we explain below. 
Consider a data point $\x$ with a one-hot label $\y$.
With the CE loss in \eqref{eq:ce-loss} rewritten below for convenience:
\begin{equation}
    L_{\mathrm{CE}}(\mb \theta,\u, \v; \x, \y) \doteq \ell_{\mathrm{CE}} \Big( \phi\big(f(\x, \mb \theta) + \s \big), \,\y \Big), ~\text{with}~\s \doteq \u \odot \u \odot \y - \v \odot \v \odot (1-\y),
\end{equation}
we may compute its gradient with respect to (w.r.t.) $\v$ as
\begin{equation}
    \frac{\partial L_{\mathrm{CE}}(\mb \theta,\u, \v; \x, \y)}{\partial \v} = \frac{2 \v \odot (1 - \y)}{\1^\top (f(\x, \mb \theta) + \s)}. 
\end{equation}
This shows that the gradient w.r.t. different entries of $\v$ does not depend on the output $f(\x, \mb \theta)$ of the model at all modulo the divider shared by all entries. 
Hence, $\v$ cannot correctly learn the label noise. 

We now consider the MSE loss in \eqref{eq:mse-loss} rewritten below for convenience:
\begin{equation}
   L_{\mathrm{MSE}}(\mb \theta,\u, \v; \x, \y) \doteq \ell_{\mathrm{MSE}} \Big( f(\x, \mb \theta) + \s, \,\y \Big), ~\text{with}~\s \doteq \u \odot \u \odot \y - \v \odot \v \odot (1-\y). 
\end{equation}
The gradient w.r.t. $\v$ can be computed as
\begin{equation}
    \frac{\partial L_{\mathrm{MSE}}(\mb \theta,\u, \v; \x, \y)}{\partial \v} = 4 (f(\x, \mb \theta) + \s - \y) \odot \v \odot (1-\y).
\end{equation}
Here the gradient w.r.t. different entries of $\v$ varies depending on how well the model prediction $f(\x, \mb \theta) + \s$ matches the given label $\y$ at the corresponding entry. 
Hence, when the model prediction deviates from the given label which may occur when the label is corrupted, $\v$ is able to learn the underlying corruption to the label.

\subsection{Definition of Label Noise}
\label{app:noise-definition}
In this paper, we consider two types of widely existed label noise, namely symmetric label noise and asymmetric label noise. For symmetric noise with noise level $\alpha$, the labels are generated as follows:
\begin{equation*}
y = \left \{   
\begin{aligned}
& y^{GT} \text{  with the probability of $1 - \alpha$}\\
& \text{random one hot vector with the probability of $\alpha$}.
\end{aligned}
\right.
\end{equation*}
We consider noise level $\alpha \in \{0.2, 0.4, 0.5, 0.6, 0.8\}$. For asymmetric noise, following~\cite{patrini2017making}, we flip labels between TRUCK $\rightarrow$AUTOMOBILE, BIRD $\rightarrow$ AIRPLANE,
DEER $\rightarrow$ HORSE, and CAT $\leftrightarrow$ DOG. 
We randomly
choose 40\% training data with their labels to
be flipped according to this asymmetric labeling rule. For real world datasets, Clothing1M has noise level estimated at around 38.5\%~\cite{Song2019PrestoppingHD}, and for WebVision, the noise level is estimated to be at around 20\%~\cite{li2017webvision}.

\subsection{Implementation Details of SOP+}
\label{app:sop+}
We considered two separate regularization terms to further boost the results and stabilize training. We will describe the definitions and roles of them below:

\myparagraph{Consistency regularizer $\mathcal{L}_C$.} 
We use a regularizer $\mathcal{L}_C$ to encourage consistency of network prediction on a original image and the corresponding augmented image.
Such a regularizer is commonly used  in semi-supervised learning and label noise learning literature, see e.g., ~\cite{mixmatch,li2019learning}. 
Specifically, the consistency regularizer $\mathcal{L}_C$ is defined as the Kullback-Leibler (KL)-divergence between the softmax predictions from the images with augmentations (described in Section~\ref{app:experiment}) and the softmax predictions for the corresponding images generated with Unsupervised Data Augmentation (UDA)~\cite{Xie2020UnsupervisedDA}: 
\[
\mathcal{L}_c(\mb \theta) = \frac{1}{N} \sum_{i=1}^N D_{KL}\left(f(\x_i; \mb \theta)\parallel f(\text{UDA}(\x_i); \mb \theta)\right).
\]


\myparagraph{Class-balance regularizer $\mathcal{L}_B$.} 
We use a regularizer $\mathcal{L}_B$ to prevent the network from assigning all data points to the same class. 
Following~\cite{tanaka2018joint}, we use the prior information on the probability distribution $p$ of class labels and minimize its distance in terms of KL-divergence to the mean prediction of each batch $\mathcal{B}$:
\[
\mathcal{L}_b(\mb \theta) = \sum_{k=1}^K p_k\log\frac{p_k}{\overline{f_{k}}(\x,\mb \theta)} = -\sum_{k=1}^K p_k\log \overline{f_{k}}(\x;\mb \theta),
\]
where $ \overline{f_{k}}(\x;\mb \theta) \approx \frac{1}{|\mathcal{B}|}\sum_{\x\in\mathcal{B}} f(\x;\mb \theta)$, and $p_k$ stands for the prior probability of the $k$th class.


The final loss function for SOP+ is therefore constructed by three terms as follows
\begin{align*}
    L(\mb \theta, \{\u_i, \v_i\}) + \lambda_C \mathcal{L}_C(\mb \theta)  + \lambda_B \mathcal{L}_B(\mb \theta),
\end{align*}
where $\lambda_c, \lambda_B > 0$ are the hyper-parameters.

\subsection{Experimental Settings}
\label{app:experiment}
\myparagraph{Data processing:} For experiments on CIFAR10/100~\cite{krizhevsky2009learning} without extra techniques, we use simple data augmentations including random crop and horizontal flip following previous works~\cite{patrini2017making,liu2020early}. For SOP+, we use the default setting from unsupervised data augmentation~\cite{Xie2020UnsupervisedDA} to apply efficient data augmentation to create another view of the data for consistency training.  For Clothing-1M~\cite{xiao2015learning}, we first
resize images to $256 \times 256$, and then random crop to $224 \times 224$, following a random horizontal flip. For WebVision~\cite{li2017webvision}, we randomly crop the images into size of $227\times 227$. All images are standardized by their means and variances. 

\myparagraph{Hyper-parameters of SOP:} We adopt a SGD optimizer without weight decay for $U$ and $V$. We keep all the hyper-parameters fixed for different levels of noise. For fair comparison, we adopt two settings of hyper-parameters and architectures for SOP and SOP+. More details of hyper-parameters can be found in Table~\ref{tab:sop_hyper}. Note that the method is not very sensitive to hyper-parameters $\lambda_C$ and $\lambda_B$.

\begin{table}[t]
\begin{center}
\caption{Hyper-parameters for SOP on CIFAR-10/100, Clothing-1M and Webvision datasets.}
\resizebox{\linewidth}{!}{
\begin{tabular}{c|cc|cc|c|c}
\toprule
 & \multicolumn{2}{c|}{CIFAR-10} & \multicolumn{2}{c|}{CIFAR-100} & Clothing-1M & Webvision\\
\midrule
architecture & ResNet34 & PreAct PresNet18 & ResNet34 & PreAct PresNet18 & ResNet-50 (pretrained) & InceptionResNetV2\\
batch size & 128 & 128 & 128 & 128 & 64 & 32\\
learning rate (lr) & 0.02 & 0.02& 0.02 & 0.0 2&  0.002 & 0.02\\
lr decay & 40th \& 80th & Cosine Annealing & 40th \& 80th & Cosine Annealing & 5th  & 50th\\
weight decay (wd) & $5\times 10^{-4}$  & $5\times 10^{-4}$ & $5\times 10^{-4}$ & $5\times 10^{-4}$ & $1\times 10^{-3}$ & $5\times 10^{-4}$\\
training epochs & 120 & 300 & 150 & 300 & 10 & 100\\
training examples & 45,000 & 50,000 & 45,000 & 50,000 & 1,000,000 & 66,000\\
\multirow{2}{*}{lr for $\{\u_i, \v_i\}$} & Sym: $\alpha_u = 10$, $\alpha_v = 10$ & Sym: $\alpha_u = 10$, $\alpha_v = 10$ & Sym: $\alpha_u = 1$, $\alpha_v = 10$ & Sym: $\alpha_u = 1$, $\alpha_v = 10$ & \multirow{2}{*}{$\alpha_u = 0.1$ , $\alpha_v = 1$} & \multirow{2}{*}{$\alpha_u = 0.1$ , $\alpha_v = 1$}\\
& Asym: $\alpha_u = 10$, $\alpha_v = 100$ & Asym: $\alpha_u = 10$, $\alpha_v = 100$ & Asym: $\alpha_u = 1$, $\alpha_v = 100$ & Asym: $\alpha_u = 1$, $\alpha_v = 100$ &  & \\
wd for $\{\u_i, \v_i\}$ & 0  & 0  & 0 & 0 & 0   & 0 \\
init. std for $\{\u_i, \v_i\}$ & $10^{-8}$  & $10^{-8}$ & $10^{-8}$ & $10^{-8}$&$10^{-8}$&$10^{-8}$\\
$\lambda_C$ & 0.0 & 0.9 & 0.0 & 0.9 & 0.0 & 0.0\\
$\lambda_B$ & 0.0 & 0.1 & 0.0 & 0.1 & 0.0 & 0.0\\
\bottomrule
\end{tabular}}
\end{center}
\label{tab:sop_hyper}
\end{table}
\

\section{Proofs for Theoretical Analysis with Linear Models}
\label{app:proof}
\subsection{Proof of Proposition~\ref{thm:landscape}}

We first present a simple but useful lemma.
\begin{lemma}\label{thm:critical-but-non-global}
Let $(\mb \theta, \u, \v)$ be a critical point to \eqref{eq:objective-linear} that is not a global minimum, i.e.,
\begin{align*}
  \mb r \;:=\; \mb J \mb \theta + \u \odot \u - \v \odot \v - \y \neq \mb 0.
\end{align*}
Then there exists an index $i$ such that 
\begin{align}
u^i = v^i = 0, \ r^i \neq 0,
\label{eq:prop-critical-point}
\end{align}
where $u^i$, $v^i$, and $r^i$ denote the $i$-th elements of $\u$, $\v$ and $\mb r$, respectively. 
\end{lemma}
\begin{proof}
We may compute the gradient of the objective function $h$ in \eqref{eq:objective-linear} as
\begin{align*}
    \nabla_{\mb \theta} h(\mb \theta, \u, \v) &= \mb J^\top \mb r,\\
    \nabla_{\u} h(\mb \theta, \u, \v) &=  2 \mb r \odot \u,\\
    \nabla_{\v} h(\mb \theta, \u, \v) &= - 2  \mb r \odot \v.
\end{align*}
Since $\mb r \neq \mb 0$ but $\nabla_{\u} h(\mb\theta, \u, \v) = \nabla_{\v} h(\mb\theta, \u, \v) = \mb 0$, we must have $u^i = v^i = 0$ and $r^i \neq 0$ for some $i$.
\end{proof}

We now prove Proposition~\ref{thm:landscape} as follows.
\begin{proof}[Proof of Proposition~\ref{thm:landscape}] 
We compute the hessian $\nabla^2 h$ of the objective function $h$ in \eqref{eq:objective-linear} as
\begin{align*}
    \nabla^2 h(\mb \theta,\mb u,\mb v) \;=\; 
    \begin{bmatrix}
    \mb J^\top \mb J & 2 \mb J^\top \diag (\mb u)  & - 2 \mb J^\top \diag (\mb v) \\
    2  \diag (\mb u) \mb J & 2 \diag\paren{ \mb r  +2 \mb u \odot \mb u } &- 4 \diag(\mb v \odot \mb u) \\
    - 2  \diag (\mb v) \mb J & - 4 \diag(\mb v \odot \mb u) & - 2 \diag \paren{ \mb r - 2 \mb v \odot \mb v }
    \end{bmatrix}.
\end{align*}
For any direction $\d = \begin{bmatrix} \d_\theta^\top & \d_{\u}^\top & \d_{\v}^\top\end{bmatrix}^\top$, the quadratic form of the Hessian $\nabla^2 h$ along this direction is given by
\hspace{-0.3in}
\begin{multline}\label{eq:hessian-along-direction}
    \d^\top \nabla^2 h(\mb \theta, \u, \v)\d 
    = \|\mb J\d_{\mb \theta}\|_2^2 + 4\|\u\odot \d_{\u}\|_2^2 + 4\|\v\odot \d_{\v}\|_2^2\\
      + 2\innerprod{ \mb r }{\d_{\u} \odot \d_{\u} - \d_{\v} \odot \d_{\v}}
      + 4\innerprod{\mb J \mb d_{\mb \theta} }{\u \odot \d_{\u} - \v \odot \d_{\v}} - 8\innerprod{\u \odot \d_{\u}}{\v \odot \d_{\v}}.
\end{multline}

We now consider an arbitrary critical point $(\mb \theta, \u, \v)$ of \eqref{eq:objective-linear} that is not a global minimum.
By Lemma~\ref{thm:critical-but-non-global}, there exists an $i$ such that $r^i \ne 0$ while $u^i = v^i = 0$. 
We divide the discussion into two cases.
\begin{itemize}[leftmargin=*]
    \item \textbf{Case 1:} $r^i >0$. We set $\d_{\theta} = \mb 0, \d_{\u} = \mb 0$, and $\d_{\v}$ to be such that all of its entries are zero except for the $i$-th entry which is given by $\d_{\v}^i = 1$.
    Plugging this direction into \eqref{eq:hessian-along-direction}, we obtain
    \begin{align*}
        \d^\top \nabla^2 h(\theta, \u, \v)\d 
        = 4 \underbrace{ [v^i]^2}_{v^i=0} [d_{\mb v}^i]^2  - 2 r^i  \underbrace{ [d_{\mb v}^i]^2 }_{ d_{\mb v}^i =1  } \;=\; - 2 r^i <0
    \end{align*}
    \item \textbf{Case 2:} $r^i <0$. We set $\d_{\theta} = \mb 0, \d_{\v} = \mb 0$, and $\d_{\u}$ to be such that all of its entries are zero except for the $i$-th entry which is given by $\d_{\u}^i = 1$.
    Plugging this direction into \eqref{eq:hessian-along-direction}, we obtain
    \begin{align*}
        \d^\top \nabla^2 h(\theta, \u, \v)\d 
        = 4 \underbrace{ [u^i]^2}_{u^i=0} [d_{\mb u}^i]^2  + 2 r^i  \underbrace{ [d_{\mb u}^i]^2 }_{ d_{\mb u}^i =1  } \;=\; 2 r^i <0
    \end{align*}
\end{itemize}
In both cases above we have constructed a direction of negative curvature, hence $(\mb \theta, \u, \v)$ is a strict saddle.

\end{proof}

\subsection{Proof of Proposition~\ref{thm:implicit-bias}}

The proof is based on the following lemma which follows trivially from KKT conditions:

\begin{lemma}[KKT condition]
Given any $\mb J$ and $\y$, if there exists $(\wh{\mb \theta}, \wh{\s}, \wh{\mb \nu})$ satisfying 
\begin{equation}\label{eq:optimality-linear}
\begin{split}
 \y &= \mb J \wh{\mb \theta} + \wh{\s},\\
    \wh{\mb \theta} &= \mb J^\top \wh{\mb \nu}, ~~\text{and}\\
    \wh{\mb \nu} &\in \lambda \mathrm{sign}(\wh{\s}),   
\end{split}
\end{equation}
then $(\wh{\mb \theta}, \wh{\s})$ is an optimal solution to \eqref{eq:convex-linear}.
In above, $\mathrm{sign}(z)$ is defined entrywise on $z$ as
\begin{equation}
    \text{sign}(z) = 
    \begin{cases}
    z / |z| & ~\text{if}~\wh{z} \ne 0,\\
    [-1, 1]  & ~\text{if}~\wh{z} = 0.
    \end{cases}
\end{equation}
\end{lemma}

\begin{proof}[Proof of Proposition~\ref{thm:implicit-bias}]
We divide the proof into two parts.

\paragraph{Global convergence.}
In this part we show that $\big(\mb \theta_\infty(\gamma, \alpha), \u_\infty(\gamma, \alpha), \v_\infty(\gamma, \alpha)\big)$ is a global solution to \eqref{eq:objective-linear} for any fixed $(\gamma, \alpha)$. 
Denote
\begin{equation}
    \r_\infty(\gamma, \alpha) \doteq \lim_{t \to \infty} \r_t(\gamma, \alpha).
\end{equation}
It follows from \eqref{eq:def-residual} and \eqref{eq:def-gradient-flow-infty} that the limit $\r_\infty(\gamma, \alpha)$ exists and can be written as
\begin{equation}
    \r_\infty(\gamma, \alpha) =  \J  \mb \theta_\infty(\gamma, \alpha) + \u_\infty(\gamma, \alpha) \odot \u_\infty(\gamma, \alpha) - \v_\infty(\gamma, \alpha) \odot \v_\infty(\gamma, \alpha). 
\end{equation}
Suppose for the purpose of obtaining a contradiction that $\big(\mb \theta_\infty(\gamma, \alpha), \u_\infty(\gamma, \alpha), \v_\infty(\gamma, \alpha)\big)$ is not a global solution to \eqref{eq:objective-linear}. 
It follows from Lemma~\ref{thm:critical-but-non-global} that there exists an $i$ such that 
\begin{equation}\label{eq:global-optimal-by-contradiction}
    u_\infty^i(\gamma, \alpha) = v_\infty^i(\gamma, \alpha) = 0, ~~\text{and}~~ r_\infty^i(\gamma, \alpha) \ne 0.
\end{equation}
Without loss of generality we assume that $C \doteq r_\infty^i(\gamma, \alpha) > 0$ so that $r^i_t(\gamma, \alpha) \to C$ with $t \to \infty$.  
For any $\epsilon \in (0, C)$, there exists a $t_0 > 0$ such that
\begin{equation}\label{eq:bound-on-r}
    C - \epsilon \le r^i_t(\gamma, \alpha) \le C + \epsilon, ~~\forall t > t_0.
\end{equation}
It follows from \eqref{eq:bound-on-r} and \eqref{eq:def-nu} that
\begin{multline}
    \nu_t^i(\gamma, \alpha) = - \int_0^t r^i_\tau(\gamma, \alpha) d\tau = \nu_{t_0}^i(\gamma, \alpha) - \int_{t_0}^t r^i_\tau(\gamma, \alpha) d\tau\\
    \in \Big( \nu_{t_0}^i(\gamma, \alpha) - (C + \epsilon)(t - t_0), \;\nu_{t_0}^i(\gamma, \alpha) - (C - \epsilon)(t - t_0)\Big), ~~\forall t > t_0.
\end{multline}
Using this bound on $\nu_t^i(\gamma, \alpha)$ in \eqref{eq:gradient-flow-equation}, we obtain
\begin{equation}
    v_t^i(\gamma, \alpha) = \gamma \exp\Big(-2\alpha \nu_t^i(\gamma, \alpha)\Big) \ge \gamma \exp\Big(-2\alpha \nu_{t_0}^i(\gamma, \alpha)\Big) \exp\Big(2 \alpha (C - \epsilon)(t - t_0)\Big), ~~\forall t > t_0.
\end{equation}
Taking the limit of $t \to \infty$, we obtain $v_\infty^i(\gamma, \alpha) = \infty$ which contradicts $v_\infty^i(\gamma, \alpha) = 0$ in \eqref{eq:global-optimal-by-contradiction}. 
Therefore, we conclude that $\big(\mb \theta_\infty(\gamma, \alpha), \u_\infty(\gamma, \alpha), \v_\infty(\gamma, \alpha)\big)$ is a global solution to \eqref{eq:objective-linear}.

\paragraph{Implicit regularization. }In this part we prove that $(\wh{\mb \theta}, \wh{\s})$ is an optimal solution to the regularized convex optimization problem in \eqref{eq:convex-linear}.
Let $\mb \nu_\infty(\gamma, \alpha)$ be the limit of $\mb \nu_t(\gamma, \alpha)$ in \eqref{eq:def-nu} at $t \to \infty$, and let
\begin{equation}\label{eq:def-nu-hat}
    \wh{\mb \nu} \doteq \lim_{\gamma \to 0} \mb \nu_\infty(\gamma, \alpha(\gamma)),
\end{equation}
with $\alpha(\gamma)$ defined in \eqref{eq:implicit-bias-parameter-relation}. 
We only need to show that the triplet $(\wh{\mb \theta}, \wh{\s}, \wh{\mb \nu})$ with $\wh{\mb \theta}$ defined in \eqref{eq:def-theta-hat}, $\wh{\s}$ defined in \eqref{eq:def-s-hat} and $\wh{\mb \nu}$ defined in \eqref{eq:def-nu-hat} satisfies the KKT conditions in \eqref{eq:optimality-linear}.

\begin{enumerate}[leftmargin=*]
\item Because $\big(\mb \theta_\infty(\gamma, \alpha), \u_\infty(\gamma, \alpha), \v_\infty(\gamma, \alpha)\big)$ is a global solution to \eqref{eq:objective-linear}, we have
\begin{equation}
   \mb J \mb \theta_\infty(\gamma, \alpha) + \u_\infty(\gamma, \alpha) \odot \u_\infty(\gamma, \alpha) - \v_\infty(\gamma, \alpha) \odot \v_\infty(\gamma, \alpha) = \y, ~~\forall \gamma > 0, \alpha > 0.
\end{equation}
Taking the limit of $\gamma \to 0$ with $\alpha = \alpha(\gamma)$ and noting the assumption that all limits in \eqref{eq:def-theta-hat} exist, we obtain
\begin{equation}
   \mb J \wh{\mb \theta} + \wh{\u} \odot \wh{\u} - \wh{\v} \odot \wh{\v}= \y.
\end{equation}
Plugging in the definition of $\wh{\s}$ in \eqref{eq:def-s-hat}, we obtain
\begin{align*}
    \mb y \;=\; \mb J \wh{\mb \theta} + \wh{\s}.
\end{align*}

\item By taking the limit of the relation $\mb \theta_t(\gamma, \alpha) = \J^\top \mb \nu_t(\gamma, \alpha)$ in \eqref{eq:gradient-flow-equation} and noting the assumptions that all relevant limits exist, we obtain
\begin{equation}
    \wh{\mb \theta} = \lim_{\gamma \to 0} \lim_{t \to \infty} \mb \theta_t(\gamma, \alpha(\gamma)) = \lim_{\gamma \to 0} \lim_{t \to \infty} \mb J^\top \mb \nu_t(\gamma, \alpha(\gamma)) = 
    \mb J^\top \wh{\mb \nu}.
\end{equation}

\item 
Denote $\s_\infty(\gamma, \alpha) \doteq \u_\infty(\gamma, \alpha) \odot \u_\infty(\gamma, \alpha) - \v_\infty(\gamma, \alpha) \odot \v_\infty(\gamma, \alpha)$.
By \eqref{eq:gradient-flow-equation}, we have
\begin{equation}\label{eq:prf-e-entrywise}
    s^i_\infty(\gamma, \alpha) \doteq u^i_\infty(\gamma, \alpha)^2 - v^i_\infty(\gamma, \alpha)^2
    = \gamma^2 \exp(4\alpha \nu_\infty^i(\gamma, \alpha)) - \gamma^2 \exp(-4\alpha \nu_\infty^i(\gamma, \alpha)).
\end{equation}
For each entry of $\wh{\s} = \lim_{\gamma \to 0} \s_\infty(\gamma, \alpha(\gamma))$ (recall that $\alpha(\gamma)$ is defined in \eqref{eq:implicit-bias-parameter-relation}), we may have three cases:
\begin{itemize}[leftmargin=*]
    \item \textbf{Case 1:} $\wh{s}^i > 0$. From \eqref{eq:prf-e-entrywise}, we must have $ \alpha(\gamma) \nu_\infty^i(\gamma, \alpha(\gamma))\rightarrow +\infty$ as $\gamma \rightarrow 0$ so that
    \begin{equation}
    \lim_{\gamma \to 0}\exp\Big(4\alpha(\gamma) \nu_\infty^i(\gamma, \alpha(\gamma))\Big) = \infty, ~~\text{and}~~ \lim_{\gamma \to 0}\exp\Big(-4\alpha(\gamma) \nu_\infty^i(\gamma, \alpha(\gamma))\Big) = 0.
    \end{equation} 
    Hence,
    \begin{equation}
    \begin{split}
        & \lim_{\gamma \to 0} \gamma^2 \exp\Big(4\alpha(\gamma) \nu^i_\infty(\gamma, \alpha(\gamma))\Big) = \wh{s}^i\\
        \implies \quad& \lim_{\gamma \to 0}2 \log \gamma + 4\alpha(\gamma) \nu^i_\infty(\gamma, \alpha(\gamma)) = \log \wh{s}^i\\
        \implies \quad & \lim_{\gamma \to 0} \nu^i_\infty(\gamma, \alpha(\gamma)) = \lim_{\gamma \to 0} \left( \frac{\log \wh{s}^i}{4\alpha(\gamma)} -   \frac{\log \gamma}{2 \alpha(\gamma)} \right).
    \end{split}
    \end{equation}
    Plugging in the relation $\alpha(\gamma) = -\frac{\log \gamma}{2 \lambda}$ in \eqref{eq:implicit-bias-parameter-relation}, we have
    \begin{align*}
        \lim_{\gamma \to 0} \nu^i_\infty(\gamma, \alpha(\gamma)) = - \lim_{\gamma \to 0}  \frac{  \lambda \log \wh{s}^i}{2\log \gamma} +  \lambda = \lambda.
    \end{align*}
    
    \item \textbf{Case 2:} $\wh{s}^i < 0$.
    Similar to case 1, from \eqref{eq:prf-e-entrywise} we must have     
    \begin{equation}
        \lim_{\gamma \to 0}\exp\Big(4\alpha(\gamma) \nu_\infty^i(\gamma, \alpha(\gamma))\Big) = 0, ~~\text{and}~~\lim_{\gamma \to 0}\exp\Big(-4\alpha(\gamma) \nu_\infty^i(\gamma, \alpha(\gamma))\Big) = \infty.
    \end{equation} 
    Hence,
    \begin{equation}
    \begin{split}
        & \lim_{\gamma \to 0} -\gamma^2 \exp\Big(-4\alpha(\gamma) \nu^i_\infty(\gamma, \alpha(\gamma))\Big) = \wh{s}^i\\
        \implies \quad& \lim_{\gamma \to 0}2 \log \gamma - 4\alpha(\gamma) \nu^i_\infty(\gamma, \alpha(\gamma)) = \log (-\wh{s}^i)\\
        \implies \quad&\lim_{\gamma \to 0} \nu^i_\infty(\gamma, \alpha(\gamma)) = \lim_{\gamma \to 0} \paren{ - \frac{\log (-\wh{s}^i)}{4\alpha(\gamma)} + \frac{\log \gamma}{2 \alpha(\gamma)} }.
    \end{split}
    \end{equation}
    Plugging in the relation $\alpha(\gamma) = -\frac{\log \gamma}{2 \lambda}$ in \eqref{eq:implicit-bias-parameter-relation}, we have 
    \begin{align*}
         \lim_{\gamma \to 0} \nu^i_\infty(\gamma, \alpha(\gamma)) \;=\; - \lambda.
    \end{align*}
    \item \textbf{Case 3:} $\wh{s}^i = 0$.
    From \eqref{eq:prf-e-entrywise}, we must have     
    \begin{equation}
        \lim_{\gamma \to 0}\gamma^2\exp\Big(4\alpha(\gamma) \nu_\infty^i(\gamma, \alpha(\gamma))\Big) = 0, ~~\text{and}~~\lim_{\gamma \to 0}\gamma^2\exp\Big(-4\alpha(\gamma) \nu_\infty^i(\gamma, \alpha(\gamma))\Big) = 0.
    \end{equation} 
    Hence, for any small $\epsilon \in (0, 1)$, there exists $\gamma_0 > 0$ such that for all $\gamma \in (0, \gamma_0)$, we have
    \begin{equation}
    \begin{split}
        & \gamma^2 \cdot \max\left\{\exp\Big(4\alpha(\gamma) \nu_\infty^i(\gamma, \alpha(\gamma))\Big), \exp\Big(-4\alpha(\gamma) \nu_\infty^i(\gamma, \alpha(\gamma))\Big) \right\} < \epsilon\\
        \implies \quad & 2 \log \gamma + 4\alpha(\gamma) \cdot | \nu_\infty^i(\gamma, \alpha(\gamma)) | < \log \epsilon\\
        \implies \quad &| \nu_\infty^i(\gamma, \alpha(\gamma)) | < \frac{\log \epsilon}{4\alpha(\gamma)} - \frac{\log\gamma}{2\alpha(\gamma)}.
    \end{split}
    \end{equation}
    Now, plugging $\alpha(\gamma) = -\frac{\log \gamma}{2 \lambda}$ in, we have 
    \begin{align*}
        | \nu_\infty^i(\gamma, \alpha(\gamma))) | <  - \frac{ \lambda  \log \epsilon }{ 2\log \gamma  } + \lambda \;<\; \lambda.
    \end{align*}
    Therefore, we have 
    \begin{align*}
         \lim_{\gamma \to 0} \abs{ \nu^i_\infty(\gamma, \alpha(\gamma)) } \;<\; \lambda.
    \end{align*}
\end{itemize}
Synthesizing all the above three cases, we obtain:
\begin{align*}
    \wh{\mb \nu} \;\in \; \lambda \mathrm{sign} ( \wh{\s} ).
\end{align*}

\end{enumerate}
\end{proof}

\subsection{Proof of Proposition~\ref{thm:exact-recovery}}



We begin with introducing the null space property that is widely used for providing necessary and sufficient conditions for exact recovery of sparse signals in compressive sensing.
\begin{definition}[\cite{cohen2009compressed}]\label{def:nus}
We say a matrix $\A\in\Re^{m\times n}$ satisfies the null space property with constant $\rho\in(0,1)$ relative to $S\in[n]$ if
\[
\|\v_S\|_1 \le  \rho\|\v_{S^c}\|_1 \ \text{for all }\v \in \ker \A,
\]
where $\ker \A$ is the null space of $\A$.
\end{definition}
With Definition~\ref{def:nus}, we prove Proposition~\ref{thm:exact-recovery} by using the following two lemmas.
The first lemma establishes correct recovery of $(\mb \theta_\star, \s_\star)$ from \eqref{eq:convex-linear} under the null space property.

\begin{lemma}\label{thm:recovery-under-nus}
Given matrix $\mb J$ and a matrix $\mb A$ that annihilates $\mb J$ on the left (i.e. such that $\A\J = 0$).
If $\A$ satisfies the stable null space property with constant $\rho\in(0,1)$ relative to the support of $\s_\star$, 
then the solution to \eqref{eq:convex-linear} is $(\mb \theta_\star, \s_\star)$ for any $\lambda > \lambda_0$ where $\lambda_0$ is a scalar that depends only on ($\J, \mb \theta_\star, \rho$).
\end{lemma}

The second lemma shows that the null space property is satisfied under the incoherent condition in \eqref{eq:incoherence-condition}.
\begin{lemma}\label{thm:nus-under-incoherence}
Given matrix $\mb J$ and a matrix $\mb A$ that annihilates $\mb J$ on the left, if
\begin{equation}\label{eq:s-r-mu-rho}
    k^2 r \le \frac{N}{\mu(\mb J)} \left( \frac{\rho}{\rho+1} \right)^2,
\end{equation}
then $\A$ satisfies null space property with constant $\rho$ relative to any $S$ that satisfies $|S| = k$.
\end{lemma}

\begin{proof}[Proof of Proposition~\ref{thm:exact-recovery}]
Assume that the condition in \eqref{eq:incoherence-condition} is satisfied. 
Then there exists a $\rho \in (0, 1)$ such that the condition in \eqref{eq:s-r-mu-rho} holds. 
Hence, $\A$ satisfies null space property with constant $\rho$ relative to any $S$ that satisfies $|S| = k$. 
Since $\s_\star$ is $k$-sparse, we have that $\A$ satisfies null space property with constant $\rho$ relative to the support of $\s_\star$.
Then the conclusion of Proposition~\ref{thm:exact-recovery} follows from applying Lemma~\ref{thm:recovery-under-nus}.
Finally, from Lemma~\ref{thm:recovery-under-nus} we have that $\lambda_0$ is a function of ($\J, \mb \theta_\star, \rho$), wherein $\rho$ is determined by $\A$ (hence $\J$) and the associated sparsity $k$. 
Hence $\lambda_0$ can be determined with a given ($\J, \mb \theta_\star, k$).
\end{proof}

In the rest of this section we prove Lemma~\ref{thm:recovery-under-nus} and Lemma~\ref{thm:nus-under-incoherence}.

\paragraph{Proof of Lemma~\ref{thm:recovery-under-nus}.}
We first introduce the following result on a useful property of the stable null space property.

\begin{theorem}[Useful property of stable null space property]
Suppose a matrix $\A\in\Re^{m\times n}$ satisfies the null space property with constant $\rho\in(0,1)$ relative to $S\in[n]$. Then for every vector $\x$ supported on $S$, we have
\[
\|\z - \x\|_1 \le \frac{1 + \rho}{1 - \rho}(\|\z\|_1 - \|\x\|_1)
\]
for any $\z$ with $\A\z = \A \x$.
\label{thm:null-space-property}\end{theorem}
\begin{proof}[Proof of Theorem~\ref{thm:null-space-property}] Since $\A(\z -\x) = \mb 0$, i.e., $\z-\x\in \ker \A$, the null space property of $\A$ implies
\[
\|(\z-\x)_{S^c}\|_1 \le \rho \|(\z-\x)_{S^c}\|_1,
\]
which further gives that
\[
\|\z - \x\|_1 \le (1+\rho) \|(\z-\x)_{S^c}\|_1.
\]

We now use these properties to prove the main result as
\begin{align*}
    \|\z\|_1 &= \|\z_S\|_1 + \|\z_{S^c}\|_1 = \|(\z -\x +\x)_S\|_1 + \|\z_{S^c}\|_1 \\
    & \ge \|\x\|_1 - \|(\z -\x)_S\|_1  + \|(\z-\x)_{S^c}\|_1\\
    & \ge \|\x\|_1 + (1-\rho) \|(\z-\x)_{S^c}\|_1\\
    & \ge \|\x\|_1 + \frac{1+\rho}{1-\rho}\|\z - \x\|_1,
\end{align*}
where the first inequality follows because $\x$ is only supported on $S$. 
\end{proof}

We are now ready to prove Lemma~\ref{thm:recovery-under-nus}.

\begin{proof}[Proof of Lemma~\ref{thm:recovery-under-nus}]
Let $\mb J = \mb U  \mb \Sigma \mb  V^\top$ be the compact SVD of $\mb J$ and $\mb V_{\perp}$ be an orthonormal basis that complements $\mb V$.
Then, the constraint in \eqref{eq:convex-linear} is equivalent to 
\[
\A\y = \A\s, \quad \mb\theta = \V\mb \Sigma^{-1} \U^\top (\y - \s) + \V_{\perp}\h.
\]

Thus, the problem \eqref{eq:convex-linear} is equivalent to
\begin{equation}\label{eq:convex-linear-v2}
\begin{split}
        &\min_{\s,\h} \frac{1}{2} \|\V\mb \Sigma^{-1} \U^\top (\y - \s) + \V_{\perp}\h\|_2^2 + \lambda \|\s\|_1 \\
        & \text{s.t.} \quad \A\y = \A\s,
   \end{split} \end{equation}
which is further equivalent to 
\begin{equation}\label{eq:convex-linear-v3}
\begin{split}
       & \min_{\s} g(\s):=\frac{1}{2} \|\V\mb \Sigma^{-1} \U^\top (\y - \s) \|_2^2 + \lambda \|\s\|_1 \\
       &\text{s.t.} \quad \A\s_{\star} = \A\s.
\end{split}    \end{equation}

Assume $\A$ satisfies the stable null space property with constant $\rho\in(0,1)$ relative to the support of $\s_\star$. Now for any $\s$ with $\A\s_{\star} = \A\s$, by Theorem~\ref{thm:null-space-property}, we have
\begin{align*}
    \|\s\|_1 - \|\s_\star\|_1 \ge \frac{1-\rho}{1+\rho} \|\s - \s_\star\|_1,
\end{align*}
which ensures $\s = \s_\star$ if we only minimize $\|\s\|_1$. The first term in \eqref{eq:convex-linear-v3} can be written as
\begin{align*}
   \|\V\mb \Sigma^{-1} \U^\top (\y - \s) \|_2^2 & =  \|\V  \mb \Sigma^{-1} \U^\top (\s_\star - \s + \J\mb \theta_\star) \|_2^2,
\end{align*}
where
\begin{align*}
    \mb \theta_\star \;=\; \mb V \mb \Sigma^{-1} \mb U^\top (\mb y - \s_\star ).
\end{align*}
This together with the previous equation gives
\begin{align*}
    &g(\s) - g(\s_\star)\\ &\ge \lambda \frac{1-\rho}{1+\rho} \|\s - \s_\star\|_1 + \|\V\mb\Sigma^{-1} \U^\top (\s_\star - \s + \J \mb\theta_\star) \|_2^2 - \|\V\mb\Sigma^{-1} \U^\top  \J\mb \theta_\star \|_2^2\\
    & = \lambda \frac{1-\rho}{1+\rho} \|\s - \s_\star\|_1 + \|\V\mb \Sigma^{-1} \U^\top (\s_\star - \s) \|_2^2+ 2\langle \s_\star - \s, \U\mb \Sigma^{-1}\V^\top \theta_\star \rangle \\
    & \ge \lambda \frac{1-\rho}{1+\rho} \|\s - \s_\star\|_1 - 2 \|\U\mb \Sigma^{-1}\V^\top \theta_\star\|_{\infty} \|\s - \s_\star\|_1\\
    & = \left( \lambda \frac{1-\rho}{1+\rho} - 2 \|\U\mb\Sigma^{-1}\V^\top \theta_\star\|_{\infty}\right) \|\s - \s_\star\|_1.
\end{align*}
Thus, if $\lambda > \lambda_0$ with
\begin{equation}
    \lambda_0 = 2 \frac{1+\rho}{1 - \rho}  \|\U\mb\Sigma^{-1}\V^\top \mb\theta_\star\|_{\infty},
\end{equation}
we have $g(\s) - g(\s_\star) >0$ whenever $\s\neq \s_\star$.

\end{proof}

\paragraph{Proof of Lemma~\ref{thm:nus-under-incoherence}.}


\begin{proof}
Let $\mb J = \mb U  \mb \Sigma \mb  V^\top$ be the compact SVD of $\mb J$. From \eqref{eq:s-r-mu-rho} we have
\begin{equation}\label{eq:s-r-mu-rho-reformulate}
    (\rho + 1) k \sqrt{\frac{r}{N}\mu(\J)} \le \rho.
\end{equation}
Let $S \subseteq [N]$ with $|S| = k$ and $\a \in \RR^r$ be an arbitrary vector.
We have
\begin{equation}\label{eq:Ualpha_S_1norm}
    \|[\U \a]_S\|_1 = \sum_{i \in S} |\e_i^\top \U \a| = \sum_{i \in S} | \langle \U^\top \e_i, \a \rangle |
    \le k \|\a\|_2 \cdot  \max_{i \in S} \|\U^\top \e_i\|_2 \le k \|\a\|_2 \sqrt{\frac{r}{N} \mu(\J)},
\end{equation}
where the last inequality is obtained from Definition~\ref{def:incoherence}. 
In addition, we have
\begin{equation}\label{eq:Ualpha_1norm}
    \|\U\a\|_1 \ge \|\U\a\|_2 = \|\a\|_2.
\end{equation}
Combining \eqref{eq:s-r-mu-rho-reformulate}, \eqref{eq:Ualpha_S_1norm} and \eqref{eq:Ualpha_1norm}, we get
\begin{equation}
    (\rho + 1)\|[\U \a]_S\|_1 \le (\rho+1) k \|\a\|_2 \sqrt{\frac{r}{N} \mu(\J)}
    \le \rho \|\a\|_2 \le \rho \|\U\a\|_1,
\end{equation}
hence,
\begin{equation}
    \|[\U \a]_S\|_1 \le \rho \|[\U \a]_{S^c}\|_1.
\end{equation}
Noting that $\{\U \a | \a \in \RR^r\} = \ker{\A}$, this finishes the proof by Definition~\ref{def:nus}.
\end{proof}

\end{document}